\makeatletter
\def\input@path{{styles/}}
\makeatother

\documentclass[12pt]{article}

\usepackage[bibencoding=utf8,style=alphabetic,backend=biber]{biblatex}%
\usepackage{sariel_biblatex}%

\usepackage[cm]{fullpage}%
\usepackage{amsmath}%
\usepackage{amssymb}%

   \usepackage[table]{xcolor}%
   \usepackage{xcolor}%
   \usepackage{subcaption}

\usepackage[amsmath,thmmarks]{ntheorem}%

\usepackage{titlesec}%

\usepackage{mleftright}%
\usepackage{xspace}%
\usepackage{graphicx}
\usepackage{hyperref}%
\usepackage[inline]{enumitem}
\usepackage{hyperref}%
\usepackage[ocgcolorlinks]{ocgx2}
\usepackage{graphicx}

\hypersetup{%
      unicode,
      breaklinks,%
      colorlinks=true,%
      urlcolor=[rgb]{0.25,0.0,0.0},%
      linkcolor=[rgb]{0.5,0.0,0.0},%
      citecolor=[rgb]{0,0.2,0.445},%
      filecolor=[rgb]{0,0,0.4},
      anchorcolor=[rgb]={0.0,0.1,0.2}%
}

\titlelabel{\thetitle. }%

\theoremseparator{.}%

\theoremstyle{plain}%
\newtheorem{theorem}{Theorem}[section]

\newtheorem{lemma}[theorem]{Lemma}

\newtheorem{observation}[theorem]{Observation}

\theoremstyle{plain}%
\theoremheaderfont{\sf} \theorembodyfont{\upshape}%
\newtheorem*{remark:unnumbered}[theorem]{Remark}%

\newtheorem{defn}[theorem]{Definition}

\theoremheaderfont{\em}%
\theorembodyfont{\upshape}%
\theoremstyle{nonumberplain}%
\theoremseparator{}%
\theoremsymbol{\myqedsymbol}%
\newtheorem{proof}{Proof:}%

\providecommand{\emphic}[2]{}
\providecommand{\emphi}[1]{}%
\providecommand{\emphw}[1]{}%
\providecommand{\emphOnly}[1]{}%

   \definecolor{blue25emph}{rgb}{0, 0, 11}
   \definecolor{almostblack}{rgb}{0, 0, 0.3}

   \renewcommand{\emphic}[2]{\textcolor{blue25emph}{%
         \textbf{\emph{#1}}}\index{#2}}
   \renewcommand{\emphi}[1]{\emphic{#1}{#1}}
   \renewcommand{\emphw}[1]{{\textcolor{almostblack}{\emph{#1}}}}%
   \renewcommand{\emphOnly}[1]{\emph{\textcolor{blue25}{\textbf{#1}}}}

\newcommand{\myqedsymbol}{\rule{2mm}{2mm}}
\newcommand{\SarielThanks}[1]{%
   \thanks{%
      Department of Computer Science; %
      University of Illinois; %
      201 N. Goodwin Avenue; %
      Urbana, IL, 61801, USA; %
      \href{mailto:spam@illinois.edu}{sariel@illinois.edu}; %
      \url{http://sarielhp.org/}. %
   #1%
   }%
}
\newcommand{\StavThanks}[1]{%
   \thanks{%
      Department of Computer Science; %
      University of Illinois; %
      201 N. Goodwin Avenue; %
      Urbana, IL, 61801, USA; %
      \href{mailto:spam@illinois.edu}{stava2@illinois.edu}; %
      \url{https://publish.illinois.edu/stav-ashur/}. %
   #1%
   }%
}
\newcommand{\NancyThanks}[1]{%
   \thanks{%
      Department of Computer Science; %
      University of Illinois; %
      201 N. Goodwin Avenue; %
      Urbana, IL, 61801, USA; %
      \href{mailto:spam@illinois.edu}{stava2@illinois.edu}; %
      \url{https://parasollab.web.illinois.edu/people/amato/}. %
   #1%
   }%
}

\newcommand{\HLink}[2]{\hyperref[#2]{#1~\ref*{#2}}}
\newcommand{\HLinkSuffix}[3]{\hyperref[#2]{#1\ref*{#2}{#3}}}

\newcommand{\figlab}[1]{\label{fig:#1}}
\newcommand{\figref}[1]{\HLink{Figure}{fig:#1}}

\providecommand{\deflab}[1]{\label{def:#1}}
\newcommand{\defref}[1]{\HLink{Definition}{def:#1}}
\newcommand{\defrefY}[2]{\hyperref[def:#2]{#1}}

\newcommand{\seclab}[1]{\label{sec:#1}}
\newcommand{\secref}[1]{\HLink{Section}{sec:#1}}

\newcommand{\itemlab}[1]{\label{item:#1}}

\newcommand{\lemlab}[1]{\label{lemma:#1}}
\newcommand{\lemref}[1]{\HLink{Lemma}{lemma:#1}}%

\providecommand{\eqlab}[1]{}%
\renewcommand{\eqlab}[1]{\label{equation:#1}}
\newcommand{\Eqref}[1]{\HLinkSuffix{Eq.~(}{equation:#1}{)}}

\providecommand{\remove}[1]{}%
\newcommand{\Set}[2]{\left\{ #1 \;\middle\vert\; #2 \right\}}

\newcommand{\pth}[1]{\mleft(#1\mright)}%

\newcommand{\ProbC}{{\mathbb{P}}}
\newcommand{\ExC}{{\mathbb{E}}}

\newcommand{\ExCond}[2]{\ExC\!\left[%
       #1 \;\middle\vert\; #2 \right]}

\newcommand{\Prob}[1]{\ProbC\mleft[ #1 \mright]}
\newcommand{\ProbCond}[2]{\mathop{\ProbC}\!\left[%
       #1 \;\middle\vert\; #2 \right]}
\newcommand{\Ex}[1]{\ExC\mleft[ #1 \mright]}

\newcommand{\ceil}[1]{\mleft\lceil {#1} \mright\rceil}
\newcommand{\floor}[1]{\mleft\lfloor {#1} \mright\rfloor}

\renewcommand{\th}{th\xspace}

\newcommand{\profileX}[1]{\ensuremath{\mathrm{profile}\pth{#1}}}

\newlist{compactenumA}{enumerate}{5}%
\setlist[compactenumA]{topsep=0pt,itemsep=-1ex,partopsep=1ex,parsep=1ex,%
   label=(\Alph*)}%

\newlist{compactenuma}{enumerate}{5}%
\setlist[compactenuma]{topsep=0pt,itemsep=-1ex,partopsep=1ex,parsep=1ex,%
   label=(\alph*)}%

\newlist{compactenumI}{enumerate}{5}%
\setlist[compactenumI]{topsep=0pt,itemsep=-1ex,partopsep=1ex,parsep=1ex,%
   label=(\Roman*)}%

\newlist{compactenumi}{enumerate}{5}%
\setlist[compactenumi]{topsep=0pt,itemsep=-1ex,partopsep=1ex,parsep=1ex,%
   label=(\roman*)}%

\newlist{compactitem}{itemize}{5}%
\setlist[compactitem]{topsep=0pt,itemsep=-1ex,partopsep=1ex,parsep=1ex,%
   label=\ensuremath{\bullet}}%

\usepackage{color}

\newcommand{\xbeginlgox}{\begin{minipage}{1in}\begin{tabbing}
           \quad\=\qquad\=\qquad\=\qquad\=\qquad\=\qquad\=\qquad\=\kill}
        \newcommand{\xendlgox}{\end{tabbing}\end{minipage}}

\newenvironment{program}{
   \begin{minipage}{4.0in}
   \begin{tabbing}
       \ \ \ \ \= \ \ \ \= \ \ \ \ \= \ \ \ \ \= \ \ \ \ \=
      \ \ \ \ \= \ \ \ \ \= \ \ \ \ \= \ \ \ \ \=
      \ \ \ \ \= \ \ \ \ \= \ \ \ \ \= \ \ \ \ \= \kill
}{
   \end{tabbing}
   \end{minipage}
}

\DefineNamedColor{named}{AlgorithmColor}{cmyk}{0.07,0.90,0,0.34}

\numberwithin{figure}{section}%
\numberwithin{table}{section}%
\numberwithin{equation}{section}%

\DeclareFontFamily{U}{BOONDOX-calo}{\skewchar\font=45 }
\DeclareFontShape{U}{BOONDOX-calo}{m}{n}{
  <-> s*[1.05] BOONDOX-r-calo}{}
\DeclareFontShape{U}{BOONDOX-calo}{b}{n}{
  <-> s*[1.05] BOONDOX-b-calo}{}
\DeclareMathAlphabet{\mathcalb}{U}{BOONDOX-calo}{m}{n}
\SetMathAlphabet{\mathcalb}{bold}{U}{BOONDOX-calo}{b}{n}
\DeclareMathAlphabet{\mathbcalb}{U}{BOONDOX-calo}{b}{n}

\usepackage{stmaryrd}%

\newcommand*{\T}{\mathcal{T}}

\newcommand{\Opt}{\mathcal{O}}

\newcommand{\RExt}{\mathbb{R}^{+\infty}}%

\newcommand{\Distrib}{\mathcal{D}}
\newcommand{\alg}{\texttt{alg}\xspace}

\newcommand{\hprofileInner}{\mathcalb{h}}
\newcommand{\hprofileX}[1]{\hprofileInner\pth{#1}}%

\newcommand{\medianC}{\mathcalb{m}}%
\newcommand{\medianY}[2]{\medianC_{#1}\mleft[  #2 \mright]}%
\newcommand{\medianX}[1]{\medianC\mleft[ #1 \mright]}%
\newcommand{\MedCond}[2]{\medianC\!\left[%
       #1 \;\middle\vert\; #2 \right]}

\newcommand{\pr}{\mathcalb{p}}

\newcommand{\threads}{\ensuremath{\mathcalb{p}}\xspace}
\newcommand{\Threads}{\ensuremath{\mathcalb{P}}\xspace}

\newcommand{\etal}{\textit{et~al.}\xspace}

\newcommand{\Term}[1]{\textsf{#1}}
\newcommand{\TTL}{\textsf{TTL}\xspace}
\newcommand{\DOF}{\textsf{DOF}\xspace}
\newcommand{\PPL}{\textsf{PPL}\xspace}

\newcommand{\SingleRef}{\hyperref[item:single]{\Single}\xspace}
\newcommand{\ValisRef}{\hyperref[item:valis]{\Valis}\xspace}
\newcommand{\ClusterRef}{\hyperref[item:cluster]{\Cluster}\xspace}

\newcommand{\proxy}{\mathcalb{t}}%
\newcommand{\EWX}[1]{\mathcalb{w}_{#1}}%
\newcommand{\fTTLX}[1]{\mathcal{T}_{#1}}

\newcommand{\bitsX}[1]{\mathcalb{b}\pth{#1}}
\newcommand{\BIN}{\ensuremath{\textbf{BIN}}\xspace}
\newcommand{\Good}{\mathcal{G}}%

\newcommand{\Single}{\texttt{L1}\xspace}
\newcommand{\Valis}{\texttt{L12}\xspace}
\newcommand{\Cluster}{\texttt{L96}\xspace}

\newcommand{\BugTrapShort}{\texttt{BugTrap}$_{0.25}$\xspace}

\newcommand{\BugTrapLong}{\texttt{BugTrap}$_{2.0}$\xspace}

\newcommand{\Manipulator}{\textsc{Manipulator}\xspace}
\newcommand{\ManipulatorRef}{\hyperref[item:manipulator]{\Manipulator}\xspace}
\newcommand{\BugTrap}{\textsc{Bug Trap}\xspace}
\newcommand{\BugTrapRef}{\hyperref[item:bug:trap]{\BugTrap}\xspace}
\newcommand{\Maze}{\textsc{Maze}\xspace}
\newcommand{\MazeRef}{\hyperref[item:maze]{\Maze}\xspace}
\newcommand{\LongDetour}{\textsc{Long Detour}\xspace}
\newcommand{\LongDetourpRef}{\hyperref[item:long:detour]{\LongDetour}\xspace}
\newcommand{\SimplePassage}{\textsc{Simple Passage}\xspace}
\newcommand{\SimplePassageRef}{\hyperref[item:simple:passage]{\SimplePassage}\xspace}

\newcommand{\FailX}[1]{\textcolor{red}{\textbf{#1}}}

\newcommand{\PRM}{\Term{PRM}\xspace}
\newcommand{\SBMP}{\Term{SBMP}\xspace}
\newcommand{\RRT}{\Term{RRT}\xspace}
\newcommand{\Seq}{\mathcal{T}}

\newcommand{\si}[1]{}

\newcommand{\myparagraph}[1]{%
   \bigskip%
   \noindent%
   \textbf{#1} %
}

\newcommand{\obslab}[1]{\label{observation:#1}}
\newcommand{\tblref}[1]{\HLink{Table}{table:#1}}
\newcommand{\obsref}[1]{\HLink{Observation}{observation:#1}}
\newcommand{\tbllab}[1]{\label{table:#1}}

\usepackage{todonotes}

\newcommand{\urlBugTrap}{\url{https://parasollab.web.illinois.edu/resources/mpbenchmarks/BugTrap/}}

\newcommand{\DOFX}[1]{$#1$\DOF}
\newcommand{\mazeUrl}{https://parasollab.web.illinois.edu/resources/mpbenchmarks/Maze/}

\bibliography{catastrophe}%

\begin{document}
\title{Faster Sampling-Based Motion Planning via Restarts}

   \author{%
      Nancy M. Amato%
      \NancyThanks{}
      \and%
      Stav Ashur%
      \StavThanks{}
      \and%
      Sariel Har-Peled%
      \SarielThanks{Work on this paper was partially supported by NSF
         AF award CCF-2317241.  }%
   }%

   \date{\today}%

\maketitle

\begin{abstract}
    Randomized methods such as \PRM and \RRT are widely used in motion planning. However, in some cases, their running time suffers from inherent instability, leading to ``catastrophic'' performance even for relatively simple instances.  We propose stochastic restart techniques, some of them new, for speeding up randomized algorithms, that provide dramatic speedups in practice, a factor of $3$ (or more in many cases).

    Our experiments demonstrate that the new algorithms have faster runtimes, shorter paths, and greater gains from multi-threading (when compared with a straightforward parallel implementation).  We prove the optimality of the new variants. Our implementation is open source, available on GitHub \cite{ah-csula-25-real}, and is easy to deploy and use.
\end{abstract}

\section{Introduction}
\seclab{intro}

Sampling-based motion planning (\SBMP) is a family of randomized algorithms used to solve the robot motion planning problem \cite{kslo-prpp-96,l-rrtnt-98}. Their probabilistic nature enables these methods to solve many instances of the problem despite its intractability \cite{r-cmpg-79}, making them widely used in robotics.  Many \SBMP algorithms, including \RRT and \PRM, are \emphw{probabilistically complete} under reasonable assumptions \cite{kslbh-pcrgk-19, kslbh-cpcrg-23,ock-smpcr-24}, implying they are Las Vegas algorithms, that is, randomized algorithms that never fail but whose running time is a random variable.

The probabilistic nature of these algorithms causes instability in the running time, as ``incorrect'' early decisions might lead to much longer exploration until a solution is found, see \figref{non:monotone}. This instability may manifest as motion planning problems with \emphi{heavy-tailed} runtime distributions, meaning, intuitively, that the expected runtime increases significantly due to these high values.

\begin{figure}[h!]
    \centering \subfloat[]{%
       \includegraphics[page=1]{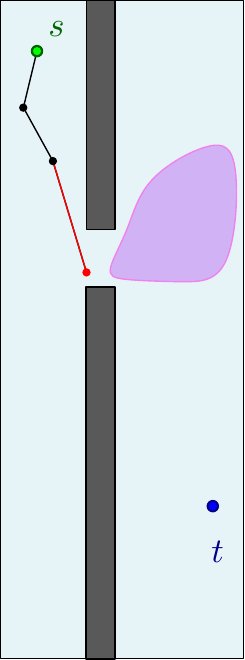}%
       \figlab{first_case} } \hfil
    \subfloat[]{\includegraphics[page=2]{figs/non_monotone}%
    } \hfil \subfloat[]{\includegraphics[page=3]{figs/non_monotone}%
    }

    \caption{An illustration of two possible \RRT search trees for a 2D point robot motion planning problem with a new extension in red. In (a), after adding the extension, the algorithm must sample the purple region and then the orange region to find the passage. The extension in (b) results in a single larger region, which will lead to the same result. In (c), we can see a possible result of an \RRT run in which the first extension was chosen.}
    \figlab{non:monotone}
\end{figure}

This problem of Las Vegas with heavy-tailed runtime distributions is not unique to robot motion planning, and tools for addressing it have been developed and successfully used in several fields.

\myparagraph{An example.} %
Consider an algorithm \alg with its running time $X$, such that $\Prob{X =1} = 0.01$ (i.e., \alg terminates after one second with probability $1/100$. Otherwise, it runs forever. (In practice, ``forever'' corresponds to cases where the algorithm just takes longer than what is acceptable, e.g., minutes instead of seconds.) The optimal strategy here is obvious -- run \alg for one second, if it terminates, voil\'{a}, a successful run was found. Otherwise, terminate it and start a fresh run, with the same threshold. The running time of this strategy has a geometric distribution $\mathrm{Geom}(1/100)$, with expected running time of $100$ seconds. A dramatic improvement over the original algorithm, which most likely would have run forever.

The goal is to strategically choose restart times for the algorithm \alg such that the strategy has a small (expected) running time, and has a ``lighter'' tail (i.e., one can prove a concentration bound on the running time until a successful run of the algorithm).  Of course, in most cases, the distribution $\Distrib$ of the running time of $\alg$ is unknown. The challenge is thus to perform efficient \emph{simulation}, that is, a restarting strategy of \alg without knowing $\Distrib$.  This problem was studied by Luby \etal \cite{lsz-oslva-93}, who presented an elegant strategy that runs in (expected) $O( \Opt \log \Opt)$ time, where $\Opt$ is the expected running time of the optimal strategy for the given algorithm. They also prove that this bound is worst-case tight.

\subsection{Contribution}

In this paper, we present several strategies for preventing catastrophic runtimes in motion planning problems. We describe the \emphi{time to live} (\TTL) simulation framework that all of the strategies share, the well-known strategy by Luby \etal \cite{lsz-oslva-93} (see \secref{counter:search:strategy}), and several new strategies for which we provide theoretical optimality analysis (see \secref{new:strategies}).  We consider several different simulation models, starting with a simple ``stop and restart'' model, and more involved schemes where the strategy creates several copies of the original algorithm and runs them in parallel, or dovetails their execution, by pausing and resuming their execution, and some schemes in between.

We demonstrate the benefits of using the different strategies with an extensive set of experiments on motion planning benchmarks (see \secref{experiments}). %
The experiments show drastically faster mean and median runtimes on tasks where \RRT has an ``unstable'' runtime distribution, i.e., a distribution with a non-negligible probability of succeeding fast, coupled with a mean and median runtime heavily influenced by much longer runtimes. We also include the results of an experiment conducted on a task where \RRT has a ``stable'' runtime distribution, meaning low variance, and verify that there is no advantage from using these strategies under such conditions.

Furthermore, we show that the strategies also greatly benefit from parallelism, and, in certain tasks, also find shorter paths than those found by plain \RRT.  In particular, in some cases, there is a factor six increase in the success rate of completing the motion planning task (on a single thread), see \tblref{shelf} (a). The effects are even more dramatic, in some cases, where the standard parallel-OR implementation, running multiple independent copies simultaneously, completely fails to complete the task, while the new strategies succeed in the majority of runs, see \tblref{bug:trap} (a).  Even in cases where the task is easier, and the parallel-OR completes all runs, the strategies exhibit a four to eight speedup, see \tblref{simple:passage}.

Our open source code implementation \cite{ah-csula-25-real} applies the presented strategies to algorithms provided via an executable.

\myparagraph{Outline.} %
The paper is organized as follows.  In \secref{related:work} we review related work on stochastic resetting and \SBMP algorithms.  In \secref{preliminaries} we discuss important preliminaries and introduce necessary notations.  In \secref{strategies} we describe the different strategies. We first explain the counter search strategy and the intuition behind it in \secref{counter:search:strategy}, and then introduce our new strategies in \secref{new:strategies}.  In \secref{experiments} we present the experiments conducted and their results. In \secref{applying:in:sbmp} we provide guidance for applying restarting strategies in \SBMP algorithms. Finally, we conclude in \secref{conclusions} with a discussion of our results.

\subsection{Related work}
\seclab{related:work}

\subsubsection{Stochastic resetting}

Alt \etal \cite{agmkw-moras-96} were the first to show how \TTL strategies can be used to transform heavy-tailed Las Vegas algorithms to randomized algorithms with a small probability of a catastrophe. Inspired by that work, which was published earlier, despite the misleading publication dates, Luby \etal \cite{lsz-oslva-93} (mentioned above) gave provably optimal strategies for speeding up Las Vegas algorithms. They showed that if the distribution is given, a fixed \TTL is optimal, and described an optimal strategy using a sequence of \TTL{}s for unknown distributions.

Luby and Ertel \cite{le-oplva-94} studied the parallel version of Luby \etal \cite{lsz-oslva-93}. They pointed out that computing the optimal strategy in this case seems to be difficult, but proved that under certain assumptions, a fixed threshold strategy is still optimal. Furthermore, they show that running the strategy of Luby \etal~on each CPU/thread yields a theoretically optimal threshold.

The simulation techniques of Luby \etal are used in solvers for satisfiability \cite{gkss-c2ss-08}, and constraint programming \cite{b-bsa-06}.  They are also used in more general search problems, such as path planning in graphs \cite{maq-acopp-23}. More widely, such strategies are related to stochastic resetting \cite{ems-sra-20}, first passage under restart \cite{pr-fpur-17}, and diffusion with resetting \cite{tpsrr-erdsr-20}. See \cite{kr-psra-24} for a more comprehensive review of this method.

Wedge and Branicky \cite{wb-htrrr-08} analyzed \RRT on a set of benchmarks, empirically determined the optimal \TTL for the tested instances, and showed an improvement in runtimes when using a good threshold. They did not suggest any strategy for the case where nothing is known about the runtime distribution of the problem.

\subsubsection{Parallel-OR \SBMP}
\seclab{parallel:or}

A widely used strategy for alleviating the possibly heavy-tailed runtime distributions of \SBMP algorithms is the parallel-OR model. First used in \SBMP by Challou \etal \cite{cgk-psarm-93}, this paradigm simply runs multiple copies of an algorithm using multithreading, and terminates as soon as the first thread succeeds. Challou \etal have shown speed-up super-linear in the number of processors, and described the mathematical reasoning of the paradigm, which was later named ``or parallel effect'' \cite{dsc-prlda-13}. Carpin and Pagello \cite{cp-prmrs-02} introduced parallel-OR \RRT{}, which was shown to have an advantage in different scenarios by Aguinaga \etal \cite{abm-prpps-08} and Devaurs \etal \cite{dsc-prlda-13}. Otte and Correll \cite{oc-cpspp-13} introduced C-FOREST, which builds upon parallel-OR \RRT and achieved super-linear speed-up for shortest path problems. Sun \etal \cite{spa-hruuu-15} have used parallel-OR \RRT to quickly replan under uncertainty and achieve shorter paths and better probability of success.

Despite the demonstrated power of the parallel-OR model, there are intuitive reasons to doubt its effectiveness in high-variance motion planning problems. Recall the example described in \secref{intro}: even $32$ or $64$ threads are not enough for a practical solution. Specifically, $64$ independent attempts at a problem instance with probability $1/100$ of ``success'', assuming the alternate outcome is unacceptable, do not amount to a reliable algorithm (i.e., in this case, the probability of failure of all $64$ runs is $\approx 52\%$).

\section{Preliminaries}
\seclab{preliminaries}

This section starts with several basic definitions, and continues by stating and proving several important claims and defining the \emph{profile} of a Las Vegas algorithm. These are required to follow the intuition behind the optimal strategy of Luby \etal, referred to as \emph{counter search} for reasons that will become apparent later, described in \secref{counter:search:strategy}, and the optimality proofs of the new strategies in \secref{new:strategies}.

\subsection{Settings and basic definitions.}
\seclab{settings:and:definitions}
\myparagraph{Time scale.}
For concreteness, we use seconds as our ``atomic'' time units, but the discussion implies interchangeability with other units of time.  In particular, we consider one second as the minimal running time of an algorithm.

Let $\RExt$ denote the set of all real non-negative numbers, including $+\infty$.  Let \alg denote the given randomized algorithm, and assume its running time is distributed according to an (unknown) distribution $\Distrib$ over $\RExt$. Given a threshold $\alpha$, and a random variable $X \sim \Distrib$, the \emphi{$\alpha$-median} of $\Distrib$, denoted by $\medianY{\alpha}{X}= \medianY{\alpha}{\Distrib}$, is the minimum (formally, infimum) of $t$ such that
\begin{math}
    \Prob{X < t} \leq \alpha \text{ and } \Prob{X > t} \leq 1-\alpha.
\end{math}
The \emph{median} of $\Distrib$, denoted by $\medianX{\Distrib}$, is the $1/2$-median of $\Distrib$.

\begin{defn}
    In the \emphi{\TTL strategy}, we are given a prespecified sequence (potentially implicit and infinite) of \TTL{}s $\Seq = t_1, t_2, t_3 \ldots$, and run \alg using these \TTL{}s one after the other. Once the running time of an execution of \alg exceeds the provided \TTL, this copy of \alg is killed, and a new copy is started using the next \TTL in the sequence.
\end{defn}

\begin{defn}
    In the \emphi{Fixed \TTL strategy}, we are given a fixed threshold $\Delta$, and run the \TTL strategy with sequence $\Seq = \Delta, \Delta, \Delta \ldots$.
\end{defn}

\begin{observation}
    \obslab{fixed:ttl:optimal}%
    The optimal strategy is always a fixed-\TTL simulation.
\end{observation}

This is due to the memorylessness of the simulation process -- the failure of the first round (which is independent of later rounds), does not change the optimal strategy to be used in the remaining rounds\footnote{This is where we use the full information assumption -- we gain no new information from knowing that the first run had failed.}. Thus, whatever the threshold being used in the second round should also be used in the first round, and vice versa. Thus, all rounds should use the same threshold.  (For a formal proof of this, see Luby \etal \cite{lsz-oslva-93}).

\subsection{Algorithms and their profile}

\begin{lemma}[Full information settings
\cite{lsz-oslva-93}]
    \lemlab{full:k}%
    Let \alg be a randomized algorithm with running time $X \sim \Distrib$ over $\RExt$. The expected running time of \alg using \TTL simulation with threshold $t$ is
    \begin{align*}
      &f(t)
        =%
        \ExCond{X}{X \leq t} + \frac{\Prob{X > t }}{\Prob{X\leq t}} t
        \leq%
        R(t)&
      \\
      &\qquad\qquad
        \text{where }%
        R(t) =
        \frac{t}{\Prob{X\leq t}}.&
    \end{align*}

    The \emphi{optimal threshold} of \alg is
    \begin{math}
        \Delta = \fTTLX{\alg} = \arg\min_{t \geq 0} f(t).
    \end{math}
    The optimal simulation of \alg, minimizing the expected running time, is the \TTL simulation of \alg using threshold $\Delta$, and its expected running time is $\Opt = f(\Delta)$.
\end{lemma}

If $\alg$ has optimal threshold $\Delta = \fTTLX{\alg}$ , and $\beta =\Prob{ X\leq \Delta}$. The optimal strategy would rerun $\alg$ (in expectation) $1/\beta$ times, and overall would have expected running time bounded by $\Delta / \beta$. Since the functions $f(t)$ and $R(t)$ of \lemref{full:k} are not the same, we need to be more careful.
\begin{defn}[Profile of an algorithm]
    \deflab{profile}%
    For the algorithm \alg, let $X$ be its running time.  The \emphi{proxy} running time of \alg with threshold $\alpha$ is the quantity $R(\alpha) = \alpha/ \Prob{X \leq \alpha}$. Let $\proxy = \arg \min_{\alpha > 0 } R(\alpha)$ be the optimal choice for the proxy running time.  The \emphi{profile} of $\alg$, is the point
    \begin{equation}
        \profileX{\alg}
        =
        (1/\pr, \proxy), \qquad\text{where}\qquad
        \pr = \Prob{X \leq \proxy}.
        \eqlab{profile}
    \end{equation}
    The total expected \emphi{work} associated with $\alg$ is $\EWX{\alg} = R(\proxy) = \proxy / \pr$.
\end{defn}

We show below that, up to a constant, the optimal strategy for \alg with profile $(1/\pr, \proxy)$ is to run the \TTL simulation with threshold $\proxy$. In expectation, this requires running \alg $1/\pr$ times, each for $\proxy$ time. This allows us later on to more easily prove the optimality of new strategies. The proof is mathematically neat, but the reader might benefit from skipping it on a first read.

\begin{lemma}
    \lemlab{equiv}%
    If \alg has profile $(1/\pr, \proxy)$, then the optimal strategy for \alg takes $\Theta(R(\proxy)) = \Theta( \proxy/\pr)$ time in expectation.
\end{lemma}

\begin{proof}
    Let $X$ be the running time of \alg. Using the notations of \lemref{full:k}, let $\Delta = \fTTLX{\alg}$, and let $O( f(\Delta) )$ be the expected running time of the optimal strategy for $\alg$. Let $\beta = \Prob{X \leq \Delta}$.  \lemref{full:k} implies that the expected running time of the \TTL simulation of \alg with \TTL $\Delta$ is
    \begin{equation*}
        f(\Delta)
        =%
        \ExCond{X}{X \leq \Delta} + \frac{1- \beta}{\beta} \Delta
        \leq
        \frac{\Delta}{\beta},
    \end{equation*}
    and $f(\Delta) \leq f(\proxy) \leq R(\proxy) = \proxy/\pr$. So $O(R(\proxy))$ is an upper bound on the expected running time of the optimal simulation of \alg.

    If $\beta \leq \tfrac78$, then $f(\Delta) \geq \tfrac18 \Delta/\beta = \tfrac18 R(\Delta) \geq \tfrac18 R(\proxy) = \Omega(\proxy/\pr)$.

    So assume $\beta > \tfrac78$. Let $\medianC = \MedCond{X }{ X \leq \Delta}$ and $\xi = \ExCond{X}{X \leq \Delta} $.  By Markov's inequality, we have that $\medianC \leq 2 \xi$. Namely, $f(\Delta) \geq \medianC/2$.  Observe that
    \begin{align*}
      \Prob{ X \leq \medianC }
      &=
      \Prob{ (X \leq \medianC) \cap (X \leq \Delta)  }
      \qquad\qquad\qquad\qquad
      \\&%
      =
      \ProbCond{ X \leq \medianC}{X \leq \Delta  } \Prob{X \leq \Delta}
      \geq
      \frac{ \beta}{2}
      >
      \frac{7}{16}.
    \end{align*}
    Thus, we have
    \begin{math}
        R(\proxy)%
        \leq%
        R(\medianC)%
        =%
        \frac{\medianC}{\Prob{X \leq \medianC}}%
        \leq%
        \frac{16}{7}\medianC%
        \leq%
        5 f(\Delta).
    \end{math}
\end{proof}

\section{Strategies}
\seclab{strategies}

A significant portion of this section is dedicated to the counter search strategy, which we show is a valuable tool for \SBMP algorithms. We thus not only describe the strategy, but also provide both intuition and mathematical analysis of the method in order to allow interested readers to more easily harness it. This strategy is tried-and-tested in many applications, and our experimental results show its performance record extends to the motion planning problem.

After introducing and discussing counter search, two new strategies are presented. We start with strategies using only the stop/start operations while simulating an algorithm, but, rather than having a known infinite sequence of \TTL{}s like counter search, our strategies draw a random \TTL at each iteration based on carefully devised distributions. These strategies perform on par with the counter search strategy, but their random nature means they can be parallelized in a fully decentralized manner without creating a shared resource, making them easy to implement and use. However, the theoretical guarantees they provide are only with high probability.

Next, we abandon the stop/start model and allow pause/resume operations that give rise to two more strategies, one simulating concurrently running an algorithm at different speeds, and one caching a subset of earlier simulations of the algorithm and resuming them rather than starting from scratch every time. The idea behind this model is a very simple one -- ``waste not''. Runs of the algorithm whose \TTL was reached without achieving success can be resumed to reach a higher \TTL, thus reusing the computational effort invested.  The experimental results reveal that these strategies are highly affected by hardware differences, and reusing the work is not enough to mitigate the overhead caused by the paused processes.

\subsection{Counter search strategy}
\seclab{counter:search:strategy}

Here we describe a strategy due to Luby \etal \cite{lsz-oslva-93} -- our exposition is different and is included here for the sake of completeness.

\subsubsection{$k$-front}

Let $\alg$ be a Las Vegas algorithm. The work quantity associated with $\alg$ can be visualized on a 2D system of axes as the area of an axis-aligned rectangle with corners at the origin and $\profileX{\alg} = (x,y)$ (viewed as a 2D point) -- see \Eqref{profile}. Thus, all the algorithms with the same amount of work $\alpha$ have their profile points on the hyperbola $\hprofileX{\alpha} \equiv( y=\alpha/x)$. See \figref{parabola} for an illustration.

\begin{figure}[h!]
    \centering%
       \includegraphics[scale=0.7]{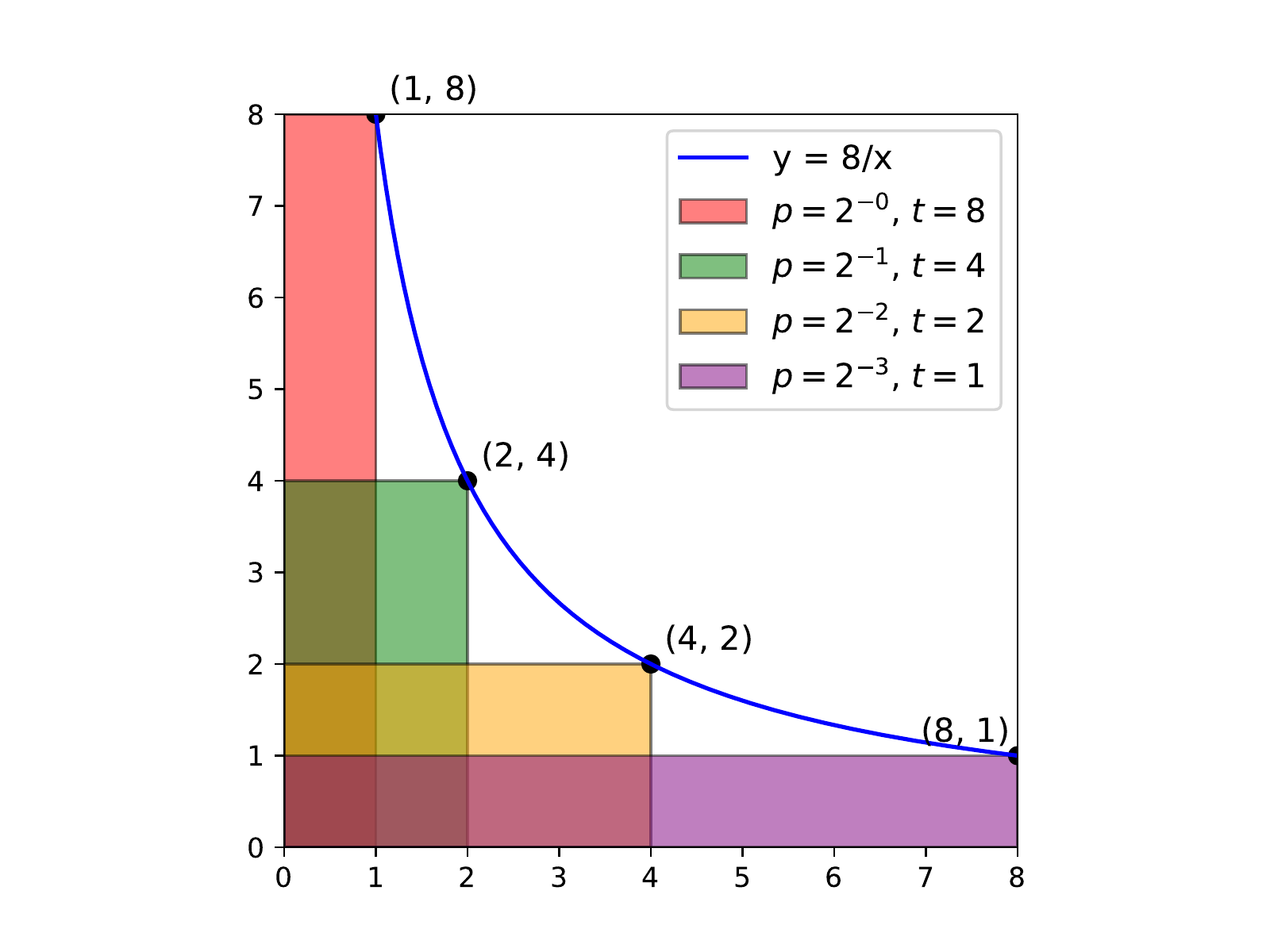}%
       \vspace{-0.5cm}
    \caption{An illustration of the work associated with four different Las Vegas algorithms with different profiles $(1/\pr, \proxy)$.  The upper right corner of every rectangle is its profile. $\pr$ and $\proxy$ are explicitly given in the legend.  }
    \figlab{parabola}
\end{figure}

A natural way to figure out all the profiles of algorithms that the prefix $t_1, \ldots, t_k$ (of a sequence $\Seq$), suffices for in expectation, i.e., a strategy for \alg should have a successful run using only these \TTL{}s), is to sort them in decreasing order, and interpret the resulting numbers $t_1' \geq t'_2, \ldots, t_k'$ as a bar graph (each bar having width $1$). All the profiles under this graph are ``satisfied'' in expectation using this prefix. We refer to the function formed by the top of this histogram as the $k$-\emphi{front} of $\Seq$. See \figref{k:front} for an example.

\begin{figure}[h!]
    \centering%
    \includegraphics[scale=0.7]{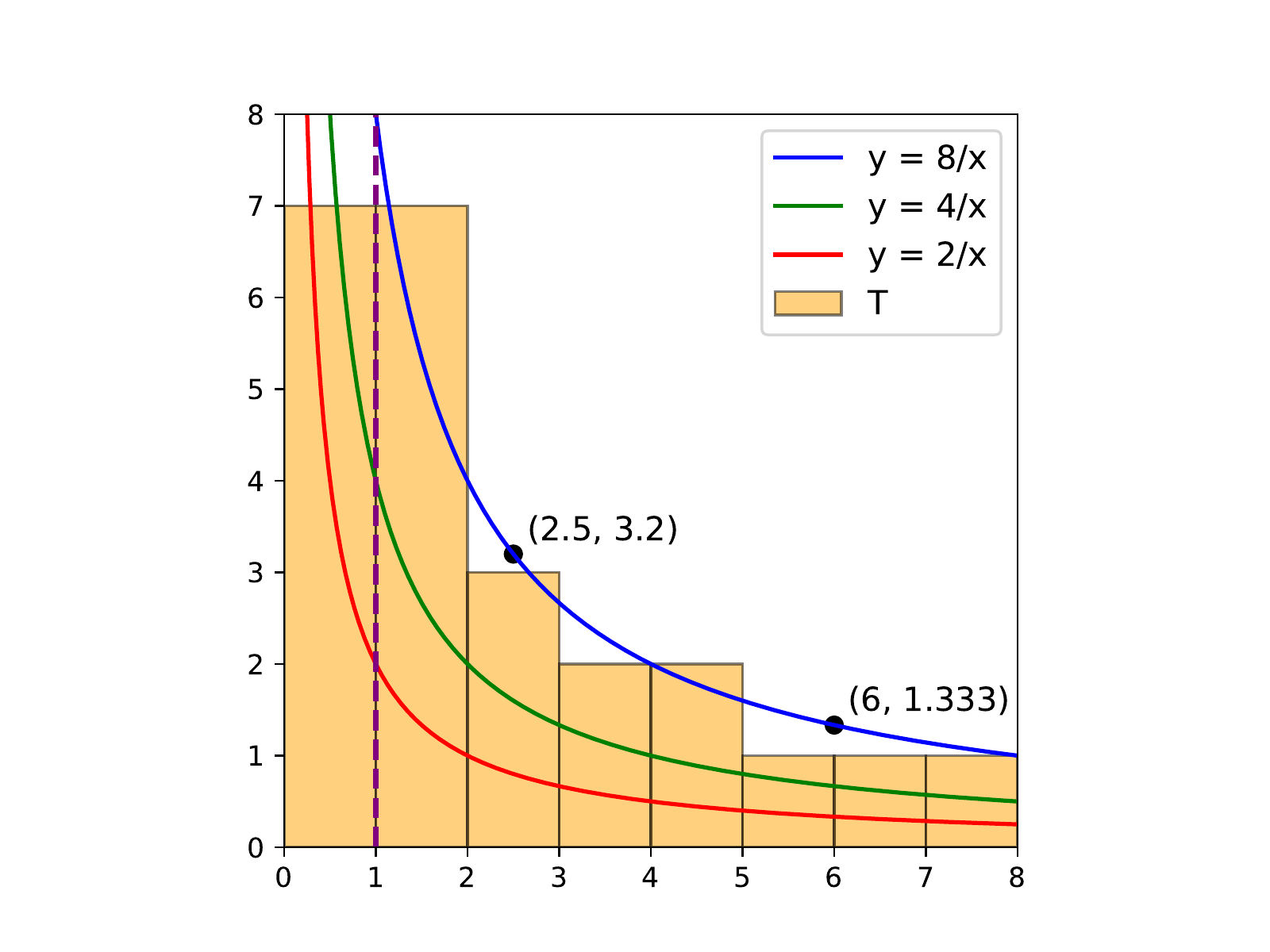}
    \vspace{-0.5cm}
    \caption{An example of an $8$-front of a sequence of \TTL{}s $\T = [7,7,3,2,2,1,1,1...]$. Three parabolas for algorithms with different associated work quantities are plotted alongside the $8$-front of $\T$. Two profiles of algorithms that, in expectation, may not have a successful run are labeled. Note that since we assume that the minimum runtime of an algorithm is one second, the part of the plot where $x<1$ is inconsequential, and the intersection of the parabola with the vertical line $x=1$ is relevant only if the algorithm is deterministic (meaning $\pr = 1$).  }
    \figlab{k:front}
\end{figure}

A key insight is that even if we knew that the profile of \alg lies on the hyperbola $\hprofileX{\alpha} \equiv( y=\alpha/x)$, we do not know where it lies on this hyperbola. Thus, an optimal sequence front should sweep over the \emphi{$\alpha$-work hyperbola}
\begin{equation*}
    \hprofileX{\alpha}
    =%
    \Set{ \Bigl. (1/\pr, t) }{ t/\pr = \alpha \text{ and } 1 \leq 1/\pr
       \leq \alpha }
\end{equation*}
more or less at the same time.

\subsubsection{The strategy}

In the following, let $\{x\}^y$ be the sequence of length $y$ made out of $x$s.  Fix a value $2^i$. We are trying to ``approximate'' the hyperbola $\hprofileX{2^i}$ with a sequence whose front roughly coincides with this hyperbola.  One way to do so is to consider a sequence $S_1$ containing the numbers $\{2^i\}^1, \{2^{i-1}\}^2,\ldots \{2^{0}\}^{2^{i+1}}$ (not necessarily in this order!).

The intuition for this choice is the following -- we might as well round the optimal threshold of the algorithm to the closest (bigger) power of $2$, and let $\nabla$ denote this number. Unfortunately, even if we knew that the work of the optimal simulation is $\alpha = 2^i$, we do not know which value of $\nabla$ to use, so we try all possible choices for this value. One can
 verify that $S_i$ dominates the work hyperbola $\hprofileX{2^i}$ (for all integral points).

\begin{figure}[h!]
    \centering%
    \includegraphics{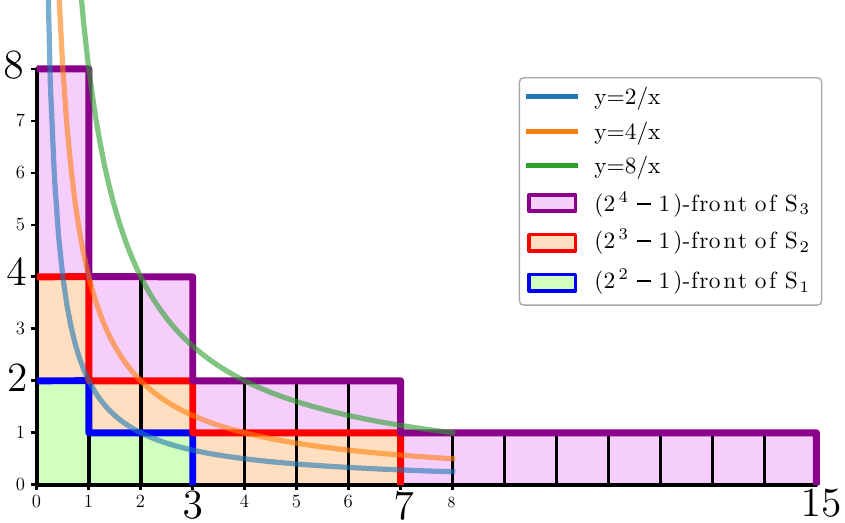}%
    \caption{Plots of three $k$-fronts of \TTL sequences generated by the counter strategy. Note that the bars ``hide'' each other, e.g., the bars of $S_3$ (light purple) have $y$-values (8, 4, 4, 2,...). In particular, the region covered by $S_1 =2,1,1$ is bounded by the blue staircase. Similarly, $S_2=4,2,2,1,1,1,1$ is bounded by the red staircase, and $S_3=8,4,4,2,2,2,2,1,\ldots,1$ is bounded by the purple staircase. }
    \figlab{luby:k:front}
\end{figure}

What is not clear is how to generate an infinite sequence whose prefix is $S_i$ if one reads the first $|S_i|$ elements of this sequence, for all $i$. As Luby \etal \cite{lsz-oslva-93} suggest, ignoring the internal ordering of $S_i$, we have that $S_{i+1} = S_i, S_i, 2^{i+1}$, where $S_0 = 1$. Indeed, every number in $S_i$ appears twice in $S_{i+1}$ except for the largest value, which is unique. We have $S_1 = 1,1,2$, $S_2 = 1,1,2,1,1,2,4$, and so on. See \figref{luby:k:front} for an illustration of $k$-fronts generated by the strategy.

Algorithmically, there is a neat way to generate this \emphi{counter strategy} \TTL sequence.  Let $c$ be an integer counter initialized to $0$. To generate the next elements in the sequence, at any point in time, we increase the value of $c$ by one, and add all the numbers $2^i$, for $i=0, 1, 2, 3, \ldots$, that divide $c$ into the sequence. Thus, the sequence is
\begin{equation*}
    T \equiv
    \underbrace{1}_{c=1},\quad \underbrace{1,2}_{c=2},\quad
    \underbrace{1}_{c=3}, \quad
    \underbrace{1,2,4}_{c=4},
    \quad
    \underbrace{1}_{c=5}, \ldots
\end{equation*}

\begin{lemma}[\cite{lsz-oslva-93}]
    Running a stop/start simulation on an algorithm \alg, using the sequence $T$ above, results in expected running time $O( \Opt \log \Opt)$, where $\Opt$ is the expected running time of the optimal simulation for \alg.
\end{lemma}
\begin{proof}
    Let $S(i,k)$ be the sum of all the appearances of $2^i$ in the first $k$ elements of $T_k$. For $2^i > 2^j$ (both appearing in the first $k$ elements), we have $S(i,k) \leq S(j,k) \leq 2S(i,k)$, as one can easily verify.

    Let $(1/\beta, \Delta)$ be the profile of $\alg$, see \defref{profile}.  The optimal simulation takes in expectation $\Opt= \Theta(\Delta/\beta)$ time.  Let $k$ be the number of elements of $T$ that were used by the simulation before it stopped.  Let $B$ be the number of elements in $t_1, \ldots, t_k$, at least as large as $\Delta$.  Clearly, $B \sim \mathrm{Geom}(\beta)$ (to be more precise, it is stochastically dominated by it). Thus, we have $\Ex{B} \leq 1/\beta$. Furthermore, the quantity $\ell = 1 +\ceil{\log_2 (B \Delta)}$ is an upper bound on the largest index $\ell$ such that $2^\ell$ appears in $t_1, \ldots, t_k$.  An easy calculation shows that $\Ex{\ell} =O(1 + \log (\Delta/\beta) ) = O( \log \EWX{\alg})$.  The total work of the simulation thus is $\sum_{i=0}^\ell S(i,k) = O( \ell B \Delta)$, which in expectation is $O( \EWX{\alg} \log \EWX{\alg} ) = O( \Opt \log \Opt)$.
\end{proof}

\subsection{New restart strategies}
\seclab{new:strategies}

In the following, we present and discuss several restarting strategies. We start by introducing randomized \TTL strategies, namely \emph{random $\zeta$ search} and \emph{random counter search}, in \secref{random:search}. These strategies randomly generate \TTL{}s whenever a new run starts instead of using a predetermined sequence of \TTL{}s. We prove that both strategies are optimal.

We then introduce two more strategies for which we abandon the ``stop/start only'' constraint, and allow the use of pausing. The \emph{wide search} strategy, discussed in \secref{wide:search}, simulates different copies of \alg running at different speeds by pausing some runs while allowing others to progress. The \emph{cache + X} strategy, which is the combination of a caching mechanism and a restarting strategy, maintains a set of paused runs of \alg that are ``extended'' to longer runs instead of simulating \alg from scratch at every step. We provide a proof of optimality only for wide search, as the caching mechanism maintains the optimality of the X component of cache + X.

\subsection{Parallelizable random search}
\seclab{random:search}

A natural approach is to generate the sequence of \TTL{}s using some prespecified distribution, repeatedly picking a value $t_i$ before starting the $i$\th run. This generates a strategy that can easily be implemented in a distributed parallel fashion without synchronization, with the additional benefit of being shockingly simple to describe.

Intuitively, our goal is to generate a distribution $\Distrib$ such that the $k$-front created by sampling $\Distrib$ $k$ times will approximate the space under the hyperbola $\hprofileX{2^{\log k - 1}}$. These distributions are asymptotically optimal and generate thresholds other than $2^i$.

\subsubsection{The Zeta $2$ distribution}
\seclab{r:zeta:2}

The Riemann zeta function is $\zeta(s) = \sum_{i=1}^\infty 1/i^s$. In particular, $\zeta(2) = \pi^2/6$ (i.e., Basel problem). The \emphi{$\zeta(2)$ distribution} over the positive integers, has for an integer $i>0$, the probability $\Prob{X=i} = c/i^2$, where $c = 6/\pi^2$.

\begin{lemma}
    \lemlab{zeta:optimality}%
    With probability $\geq 1 -1 /\Opt^{O(1)}$, the expected running time of a simulation of $\alg$, using a \TTL sequence sampled from the $\zeta(2)$ distribution is $O( \Opt \log \Opt)$, where $\Opt$ is the minimal expected running time of any simulation of \alg.
\end{lemma}
\begin{proof}
    Let $(1/\beta, \Delta)$ be the profile of \alg, and let $L = c \cdot \Opt^{c}$, where $\Opt = \Theta(\Delta/\beta)$ and $c$ is some sufficiently large constant, let $\Seq_L = t_1, t_2, \ldots, t_L \sim \zeta(2)$ be the random sequence of \TTL{}s used by the simulation, and let $X \sim \zeta(2)$.  A value $t_i \in \Seq_L$ is \emph{good} if $t_i \leq L^2$. Clearly, the probability that any of the values in $\Seq_L$ is not good is at most $L \Prob{X \geq L^2} = O(1/L)$, since for $i > 1$, we have
    \begin{align*}
      \Prob{X \geq i}
      &=%
      \sum_{\ell=i}^{\infty} \frac{6/\pi^2}{\ell^2}
      \leq%
      \frac{6}{\pi^2}\sum_{\ell=i}^{\infty} \frac{1}{\ell(\ell -1)}
        \\&%
        =%
      \frac{6}{\pi^2}\sum_{\ell=i}^{\infty} \pth{ \frac{1}{\ell-1}
      -\frac{1}{\ell}}
      =
      \frac{6}{\pi^2(i-1)}.
    \end{align*}
    A similar argument shows that
    \begin{math}
        \Prob{X \geq i} \geq \frac{6}{\pi^2i}.
    \end{math}
    The event that all the values of $\Seq_L$ are smaller than $L^2$ is denoted by $\Good$.  Observe that
    \begin{align*}
      \tau %
      &=%
        \ExCond{X }{X \leq L^2}
      =%
      \frac{1}{\Prob{X \leq L^2}} \sum_{i=1}^{L^2}
      \frac{6i}{\pi^2i^2}
        \\&%
        =
      O( \log L)
      =
      O( \log \Opt).
    \end{align*}
    Since $\Prob{t_i \geq \Delta} \geq 6/(\pi^2 \Delta)$, the probability that the algorithm terminates in the $i$\th run (using \TTL $t_i$) is at least $ \alpha = 6 \beta/(\pi^2 \Delta)$. It follows that the simulation in expectation has to perform $M = 1/\alpha = (6/\pi^2)\Opt$ rounds. Thus, the expected running time of the simulation is $\sum_{i=1}^M \ExCond{t_i}{\Good} = M \tau = O( \Opt \log \Opt)$.
\end{proof}

\subsubsection{Random counter search}
\seclab{r:counter}%

Consider a generated threshold $t$ as a binary string $b_1 b_2 \ldots, b_k$, where the numerical value of this string is the value it encodes in base two: $t = \sum_{i=1}^k b_i 2^{k-i}$.

\begin{defn}
    For a natural number $t > 0$, the number of bits in its binary representation, denoted $\bitsX{t}$, is its \emphi{length} -- that is, $\bitsX{t} = 1+\floor{\log_2 t}$.
\end{defn}

We generate a random string as follows. Let $s_1 =1$. Next, in each step, the string is finalized with probability $1/2$. Otherwise, a random bit is appended to the binary string, where the bit is randomly chosen between $0$ or $1$ with equal probability. We denote the resulting distribution on the natural numbers \BIN.  Observe that for $R \sim \BIN$, the probability that $R$ length is $k$ (i.e., $\bitsX{R} = k$) is $1/2^k$, for $k\geq 1$. Also, observe that $R$ has a uniform distribution over all the numbers of the same length.

\newcommand{\RandomCnt}{Random Counter\xspace}

Note that if we were to ``round'' every value drawn from the \BIN down to the nearest value of the form $2^i$, we would get the geometric distribution perfectly mimicking the counter search strategy. We refer to the random strategy using $\BIN$ to sample the \TTL{}s as the \emphi{\RandomCnt} simulation.

\begin{lemma}
    With probability $\geq 1 -1 /\Opt^{O(1)}$, the expected running time of the \RandomCnt simulation is $O( \Opt \log \Opt)$.
\end{lemma}
\begin{proof}
    Using the notation of \lemref{zeta:optimality} we let $(1/\beta, \Delta)$ be the profile of \alg, $L = c \cdot \Opt^{c}$, where $\Opt = \Delta/\beta$ and $c$ is some sufficiently large constant, and $T_L = t_1, t_2, \ldots, t_L$ be the random sequence used by the simulation. We also denote values $t_i \in \Seq_L$ as good if $t_i \leq L^2$. Again, we get that the probability that $\Seq_L$ contains a value which is not good is bounded by $1/L$. From this point on, we assume all the values of $T_L$ are good, and denote this event by $\Good$.

    For $R \sim \BIN$, and any integer number $t$, we have
    \begin{equation}
        \begin{aligned}
          \frac{2}{t}
          \geq
          \frac{1}{2^{\bitsX{t} -1}}
          =
          \Prob{\bitsX{R} \geq \bitsX{t}}
          \geq
          \Prob{R \geq t}\qquad\qquad
          \\%
          \geq
          \Prob{\bitsX{R} > \bitsX{t}}
          =%
          \frac{1}{2^{\bitsX{t}}}
          \geq
          \frac{1}{t}.
        \end{aligned}
        \eqlab{in:between}
    \end{equation}
    Observe that %
    \begin{math}
        \ExCond{R}{\bitsX{R} =j} =%
        (2^j + 2^{j-1}-1)/2 \leq 3 \cdot 2^{j-1}.
    \end{math}
    Furthermore, $\Prob{R \geq 2^i} = 1/2^{i-1}$, and for $i > j$, we have $\ProbCond{\bitsX{R} =j }{ R < 2^i} = 2^{-j}/(1-1/2^{i})$. Thus, we have
    \begin{align*}
      \ExCond{R}{R < 2^i}
      &=%
      \sum_{j=1}^{i-1}
      \ExCond{R}{\bitsX{R} =j}
      \ProbCond{\bitsX{R} =j }{ R < 2^i}
      \\&%
      \leq %
      \sum_{j=1}^{i-1}
      \frac{3 \cdot 2^{j-1} }{2^j(1-1/2^{i-1})}
      \leq
      3i.
    \end{align*}
    In particular, for $i=1,\ldots, L$, we have $\Ex{t_i} = O( \log \Opt )$, conditioned on $\Good$.

    A value $t_i$ of $T_L$ is \emph{final} if $t_i \geq \Delta$, and the running of the $i$\th copy of the algorithm with \TTL $t_i$ succeeded. Let $X_i$ be an indicator variable for $t_i$ being final. Observe that
    \begin{align*}
      p
      &=
      \Prob{X_i=1}
      \\&
      =
      \ProbCond{ \alg \text{ successfully runs with \TTL ~}t_i }{t_i \geq
      \Delta}
      \\&\quad\quad\cdot
      \Prob{t_i \geq \Delta}
      \\&
      \geq%
      \beta \cdot \frac{1}{\Delta}
      =
      \frac{1}{\Opt},
    \end{align*}
    by \Eqref{in:between}.

    The number of values of $T_L$ the simulation needs to use until it succeeds is a geometric variable with distribution $\mathrm{Geom}(1/\Opt)$. (Clearly, the probability that the algorithm reads all the values of $T_L$ diminishes exponentially, and can be ignored.) Thus, the algorithm in expectation reads $ 1/p = O(\Opt)$ values of $T_L$, and each such value has expectation $O( \log \Opt)$. Thus, conditioned on $\Good$, the expected running time of the simulation is $O( \Opt \log \Opt )$.
\end{proof}

\subsection{The wide search with pause/resume operations}
\seclab{wide:search}

We now remove the constraint allowing us to only start and terminate runs of \alg, and introduce the pause/resume operation, allowing us to freeze a run mid execution, and later resume it with no effect on the performed computations. Note that we still refer to all non-fixed-\TTL strategies as ``restarting strategies'' for the sake of brevity.  The \emphi{wide search} strategy uses pause/resume operations to simulate running many copies of the algorithm in parallel, but at different speeds.  Specifically, the $i$\th copy of \alg, denoted by $A_i$, is ran at speed $\alpha_i=1/i$, for $i=1,\ldots$. So, consider the simulation immediately after it ran $A_1$ for $t$ seconds --- here $t$ is the \emph{current time} of the simulation. During that time the $i$\th algorithm $A_i$ ran $n_i(t) = \floor{ t/i}$ seconds. Naturally, the simulation instantiates $A_i$ at time $i$, since $n_i(i) = 1$, and then runs it for one second. More generally, the simulation allocates a running time of one second to a copy of the algorithm $A_j$ at the time $t$, when $n_j(t)$ had increased to a new integral value. Then, $A_j$ is resumed for one second. This simulation scheme can easily be implemented using a heap, and we skip the low-level details.

It is not hard to simulate different speeds of algorithms in the pause/resume model, using a scheduler that allocates time slots for the copies of the algorithm according to their speed.

\bigskip

The wide search starts being effective as soon as the running time of the first algorithm exceeds a threshold $\nabla$ that already provides a significant probability that the original algorithm succeeds, as testified by the following lemma.

\begin{lemma}
    \lemlab{w:s:good:p}%
    The wide search simulation described above, executed on \alg with profile $(1/\beta, \Delta)$, has the expected running time $O\pth{ \frac{\nabla}{\beta} \log \frac{\nabla}{\beta}} = O( \Opt \log \Opt)$, where $\Opt$ is the optimal expected running time of any simulation of \alg (here, the $i$\th algorithm is run with speed $\alpha_i = 1/i$).
\end{lemma}
\begin{proof}
    Let $T_j = \nabla\ceil{j /\beta} $. If we run the first algorithm for $T_j$ time, then the first $\ceil{j/\beta}$ algorithms in the wide search are going to be run at least $T_j / \ceil{j/\beta} = \nabla$ time.  The probability that all these algorithms fail to successfully terminate is at most $p_j = (1- \beta)^{\ceil{j/ \beta}} \leq \exp(-j)$.  The total running time of the simulation until this point in time is bounded by
    \begin{equation*}
        S_j = \sum_{i=1}^{T_j} \frac{T_j}{i} = O(T_j \log T_j)
        =%
        O\Bigl( \frac{\nabla j}{\beta} \log \frac{\nabla j}{\beta} \Bigr).
    \end{equation*}
    Thus, the expected running time of the simulation is asymptotically bounded by
    \begin{align*}
      \sum_{j=0}^\infty p_j O(S_{j+1})
      &=
      \sum_{j=0}^\infty
      O\pth{ \exp(-j)  \frac{(j+1)\nabla}{\beta}
      \log \frac{(j+1)\nabla}{\beta}
      }
      \\&%
      = %
      O\pth{ \frac{\nabla}{\beta} \log
      \frac{\nabla}{\beta}}
      =%
      O( \Opt \log \Opt ),
    \end{align*}
    as $\Opt = O( \nabla/\beta)$.
\end{proof}

\subsection{Caching runs}
\seclab{cache}

Another natural approach that utilizes pause/resume operations is, when running any simulation, suspending running algorithms when their \TTL is met instead of terminating them, thus conserving some of the effort expended by running them.  Now, whenever the algorithm needs to be run for a certain time threshold $t_i$, we first check if there is a suspended run that can be resumed to achieve the desired threshold. To avoid encumbering the system by flooding it with suspended runs (as the wide search does), we limit the number of runs stored in the cache, and kill the ``lighter'' runs, that is, runs whose runtime is lower, when that number is exceeded.

This recycling trades some of the randomness of the original scheme for computational efficiency, as long runs of \alg are not started from scratch.

\newcommand{\Catalyst}{\texttt{catalyst}\xspace}

\section{Experiments}
\seclab{experiments}

Here, we present the experiments conducted to evaluate the advantages and weaknesses of the different restarting strategies when applied to \SBMP algorithms. We start by recounting the compared methods and providing important details about the experimental setup, followed by the description of the experiments. For every experiment, we provide the base details about the environment, robot, \SBMP algorithm, and motion planning task, the results, and a short discussion.

Our objective is to highlight the difference in performance caused by using restarting strategies. Since this framework is especially suitable for parallel implementations of \SBMP algorithms, providing a simple and powerful way to harness any number of computation nodes, we compare the different strategies against a parallel-OR benchmark, described in \secref{related:work}. To demonstrate the framework's applicability to the single process model, we first show results on using a single thread before proceeding to the main body of experiments.

Our experiments are all done in simulation, and include a 6 \emph{degrees of freedom} (\DOF) fixed-base manipulator, a \DOFX{6} free-flying object, and mobile planar robots. The first two experiments are conducted in 3D environments that give rise to \RRT queries with a high-variance runtime distribution. The second experiment is conducted in a 2D maze that gives rise to \PRM queries with a low-variance runtime distribution. These experiments aim to show the results of using stochastic resetting in scenarios with different distributions.

The last experiment, only constituting a proof-of-concept, was conducted in a 2D environment specifically designed to have two distinct paths. The goal of this experiment is to show that, in some cases, restarting strategies are expected to find shorter paths than those found by the parallel-OR strategy.

\subsection{Setup}

\subsubsection{Implementation details}

Our restarting simulator was implemented in \texttt{C++} in Linux using the standard system process mechanism. In order to not overwhelm the system with too many child processes running simultaneously. Thus, as a precaution, if the system we used has a CPU with $t$ threads (we had $t=16$ and $t=128$ in our experiments), our simulation used $(3/4)t$. We used the \SBMP algorithm implementations from the open source Parasol Planning Library (\PPL) \cite{p-pplp-25}. For both \RRT and \PRM, default \PPL parameter values were used.

Every single execution of a strategy, i.e., finding a solution for a given motion planning problem, possibly by running many copies of some algorithm \alg, was capped at $3000$ seconds. This failure threshold was selected to allow us to run a significant number of experiments on hard instances (relative to the hardware used). Throughout the tables summarizing the runtime results of the experiments, the number of failed applications of the strategy is highlighted in red. Also, the descriptive statistics in all tables refer only to successful runs.

\subsubsection{Hardware used}
To show the effect of parallelization, three hardware configurations were used. We list and provide notation for each setup below:
\begin{compactenumI}
    \item \itemlab{single}%
    \Single: Virtual Linux machine with an 8-core CPU with 8 threads (1 used), and 32GB memory.

    \medskip%
    \item \itemlab{valis} %
    \Valis: Linux machine with an \texttt{Intel i7-11700 8-core} CPU with 16 threads (12 used), and 64GB memory.  This is a standard desktop computer.%
    \medskip%
    \item \itemlab{cluster} %
    \Cluster: Linux machine with \texttt{AMD EPYC 7713 64-Core} CPU with 128 threads (96 used), and 512GB of memory.
\end{compactenumI}

\subsubsection{List of strategies}

Given an \SBMP algorithm \alg (either \RRT or \PRM) on hardware with \threads {}threads, we compare the following strategies in our experiments:

\begin{compactenumI}
    \medskip
    \item \textsc{Single}: Run a single instance of \alg using one thread.

    \medskip
    \item \textsc{Parallel}: The parallel-OR model. Run \threads independent copies of \alg using \threads threads (in parallel) until one of them succeeds. See \secref{parallel:or} for more details on the parallel-OR model.

    \medskip
    \item \textsc{Fixed-\TTL}: Every thread runs \alg until it either succeeds, or the fixed threshold \TTL is reached, in which case it restarts. See \secref{settings:and:definitions} for more details on the fixed-\TTL strategy.

    \medskip
    \item \textsc{Counter search}: A shared counter can be queried for the next \TTL in the infinite sequence. Each of the \threads{} threads queries the counter for a \TTL, and runs \alg until it either succeeds or the \TTL is reached, in which case it terminates the run and queries the counter for another \TTL. See \secref{counter:search:strategy} for more details on the counter search strategy.

    \medskip%
    \item \textsc{Wide}: A shared counter can be queried for a copy of \alg that needs to be advanced, which could either be a paused process or a new copy. Each of the \threads{} threads queries the counter, and advances the process by a number of seconds equivalent to its current total runtime. New copies of \alg are run for one second. If the process does not successfully terminate during this time, it is paused, and the thread queries the counter again. See \secref{wide:search} for more details on the wide search strategy.

    \medskip%
    \item \textsc{Random $\zeta(2)$}: Every thread draws a random number $t$ from the $\zeta(2)$ distribution, and runs a copy of \alg with \TTL $t$. In practice, this means running \threads independent copies of a random $\zeta(2)$ search. See \secref{r:zeta:2} for more details on the random $\zeta(2)$ distribution strategy.

    \medskip%
    \item \textsc{\RandomCnt}: Same strategy as random $\zeta(2)$, except that the \TTL is drawn from the random counter distributions.  See \secref{r:counter} for more details on the random counter search strategy.

    \medskip%
    \item \textsc{Counter+Cache}: Same as the counter strategy, except that the shared counter responds to queries with a \TTL $t$, and possibly a paused process $P$ that can be resumed to achieve a run of \alg with the target \TTL. The querying thread then either starts a new run of \alg with \TTL $t$, or resumes $P$ until it either succeeds or reaches $t$ seconds of runtime. If the \TTL is reached, the current process is pushed to the cache, a fixed-size priority queue based on the runtime. The cache size is \Threads, where \Threads is the maximum number of threads on the current architecture (so 16 on \Valis, and 128 on \Cluster). See \secref{cache} for more details on the X + cache strategy.

\end{compactenumI}

Note that we do not compare all strategies in every experiment. The subset of strategies used depends on the tested hypothesis. Also, to keep the number of experiments we perform under control, we only use caching in tandem with the counter search strategy, although our experimentation suggests it might be helpful in many cases.

\subsection{Inputs}
\seclab{inputs}

\begin{compactenumI}
    \item \itemlab{simple:passage} \SimplePassage
    (\figref{simple:passage}).
    We use a simple passage environment as a benchmark for high-variance cases, as both the geometries involved and the task are intuitive.
    \begin{itemize}
        \item Environment: A $10\times10\times10$ bounding box split by a wall with a single square passage with side-length 1.7 in its center.

        \item Robot: \DOFX{6} or free flying $2\times 1 \times 1$ rectangular prism.

        \item Task: Move from a starting position on one side of the wall to a configuration on the other side of the wall.

        \item Algorithm: \RRT.

    \end{itemize}

    \begin{figure}[h!]
        \centering
        \includegraphics[width=0.5\linewidth]{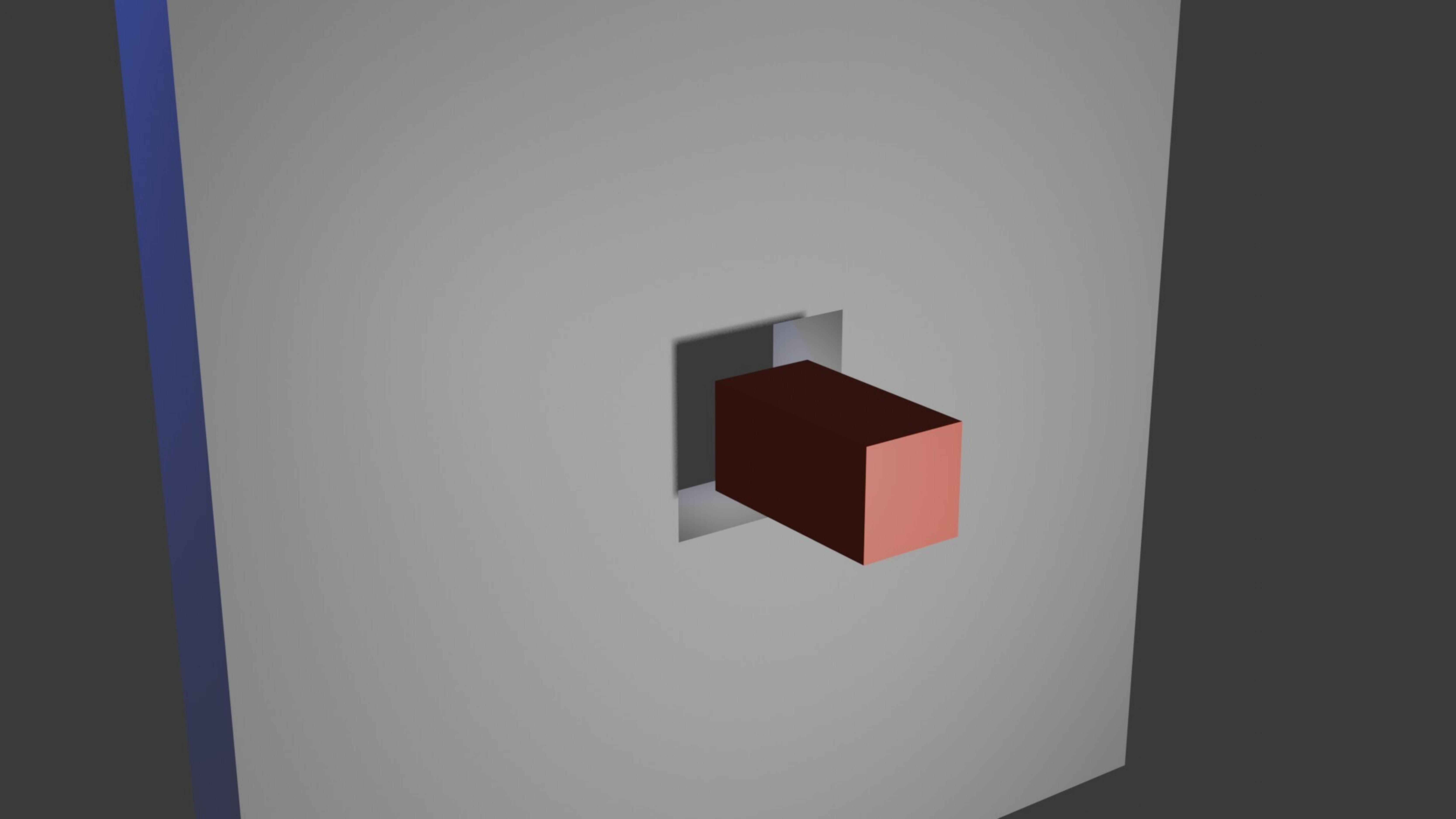}%
        \caption{The \SimplePassageRef input.  An illustration of the experiment described in \secref{simple:passage}. The robot, shown in green, fits through the passage only when properly oriented.}%
        \figlab{simple:passage}
    \end{figure}

    \smallskip%
    \item \itemlab{manipulator}%
    \Manipulator (\figref{shelf}). %
    This environment, simulating a common application for robots in residential and industrial scenarios, was chosen to demonstrate the usefulness of stochastic resetting for fixed-base manipulators.
    \begin{itemize}
        \item Environment: A table with a fixed manipulator at its center, and a double shelf is in front of the manipulator. The top shelf contains 3 objects: a cylindrical object and two rectangular prisms.

        \item Robot: \DOFX{6} manipulator.

        \item Task: Move from a neutral position outside the shelf to a configuration with the robot's end-effector inside the top shelf and its arm between two objects.

        \item Algorithm: \RRT.

    \end{itemize}

    \begin{figure}[h!]
    	\centering%
        \includegraphics[width=0.5\linewidth]{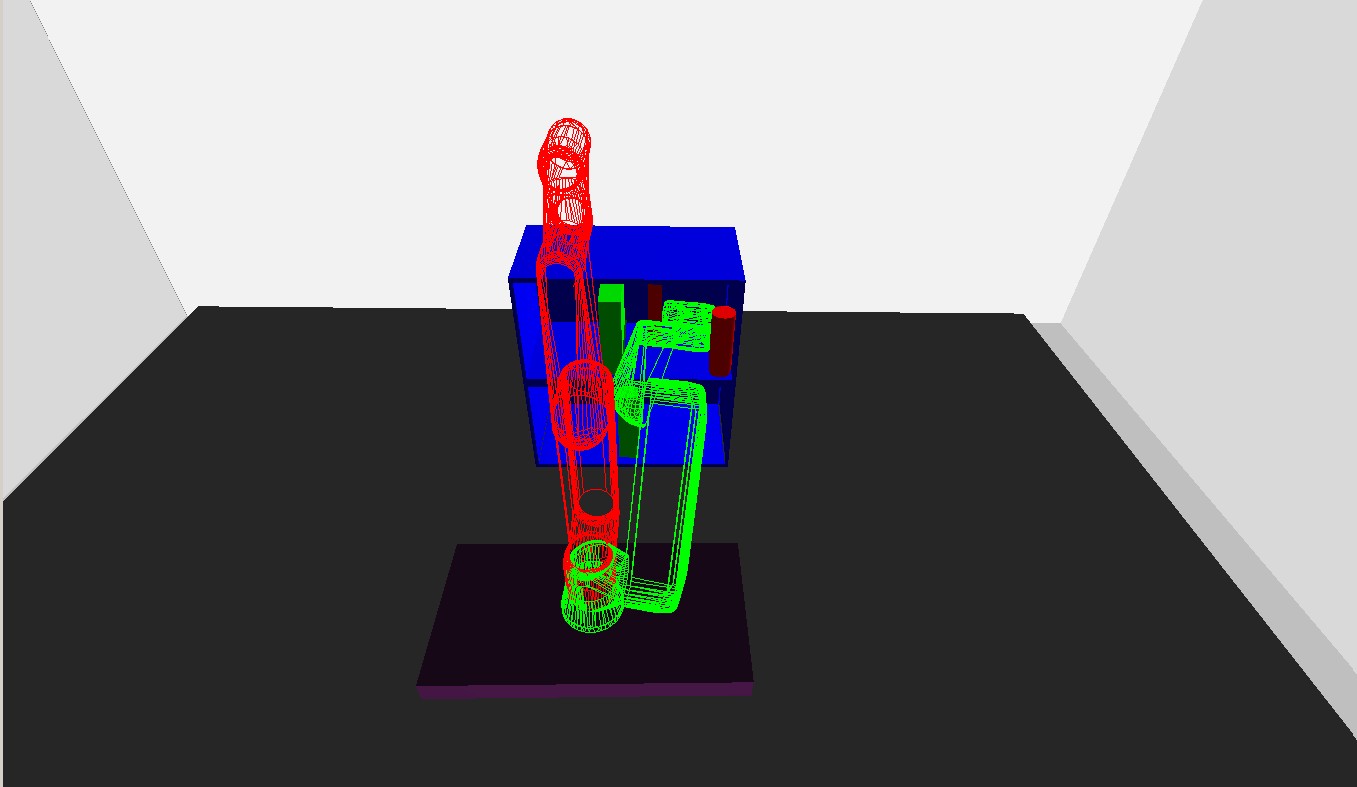}%
    	\caption{The \ManipulatorRef input.  An illustration of the experiment described in \secref{shelf}. The start and goal configurations are colored red and green, respectively.}
    	\figlab{shelf}
    \end{figure}

    \smallskip%
    \item \itemlab{bug:trap}%
    \BugTrap (\figref{bug:trap}). %
    The bug trap environment was chosen to show the usefulness of stochastic resetting for mobile free-flying robots in the presence of a narrow passage that makes the task non-trivial for \SBMP algorithms using uniform sampling.
    \begin{itemize}
        \item Environment: 3D hollow oblong shape with a single thin funnel-shaped entrance with the thin part of the funnel inside the trap.

        \item Robot: \DOFX{6} free-flying elongated cylinder.

        \item Task: Move from a starting position outside the trap to a configuration inside the trap.

        \item Algorithm: \RRT.

        \item Note: Based on the Bug Trap Benchmark%
        \footnote{\urlBugTrap} from the Parasol Lab at the University of Illinois.

    \end{itemize}

    \begin{figure}[h!]
        \centering
        \includegraphics[width=0.5\linewidth]{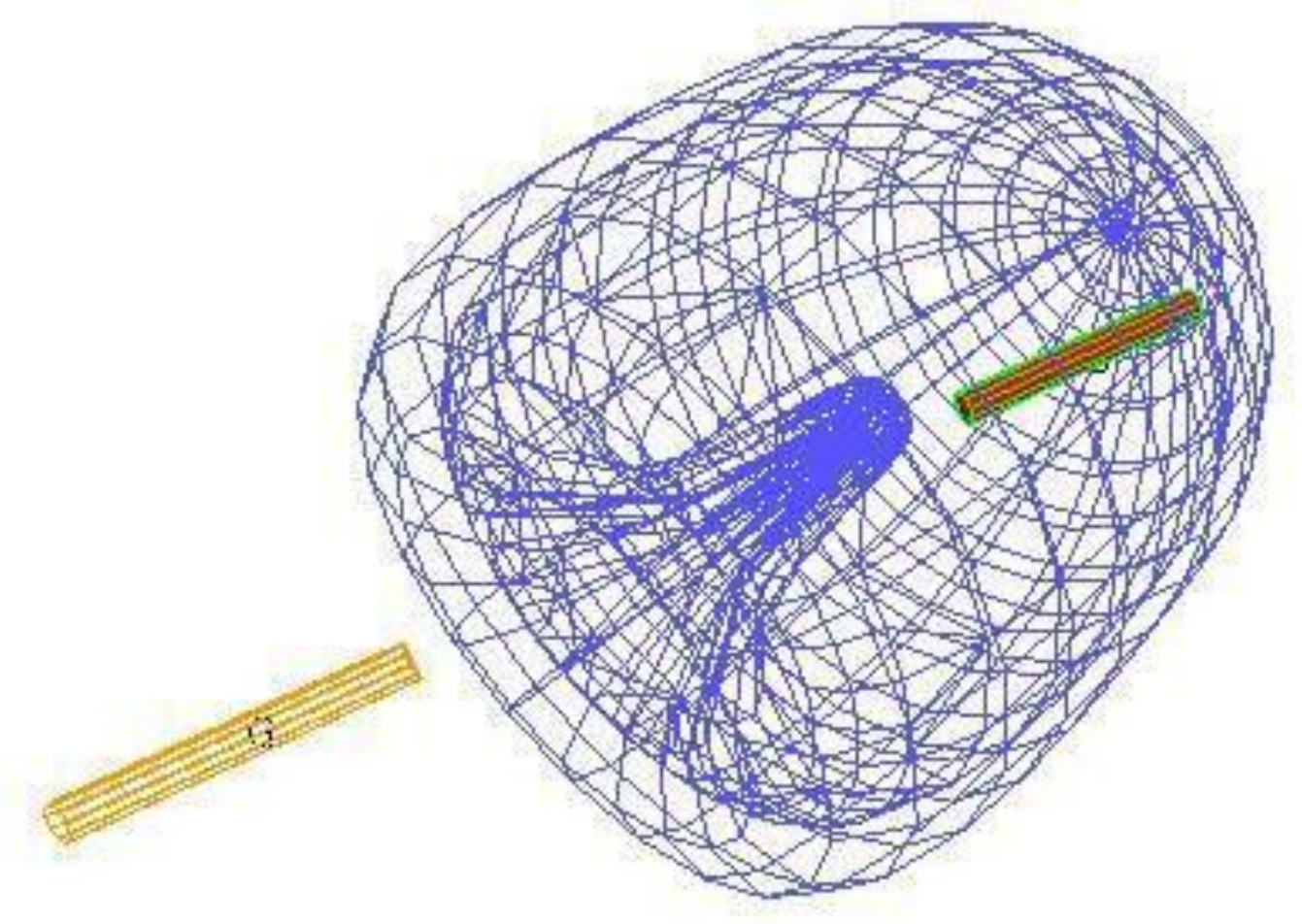}%
        \caption{An illustration of the experiment described in \secref{bug:trap}. The start and goal configurations are colored blue and red, respectively.}
        \figlab{bug:trap}
    \end{figure}

    \smallskip%
    \item \itemlab{maze}%
    \Maze (\figref{maze}).

    We chose to include this 2D environment as a ``dirty laundry'' instance, showing a task with low runtime variability.

    \begin{itemize}
        \item Environment: A 2D maze with two openings (entrance/exit).

        \item Robot: A 2D point robot.

        \item Task: Move through the maze starting at one opening and reaching the other.

        \item Algorithm: \PRM

        \item Source: Based on the \href{\mazeUrl}{Maze} Benchmark\footnote{\url{\mazeUrl}} from the Parasol Lab at the University of Illinois.

    \end{itemize}

    \begin{figure}[h!]
        \centering%
        \includegraphics[width=0.4\linewidth]{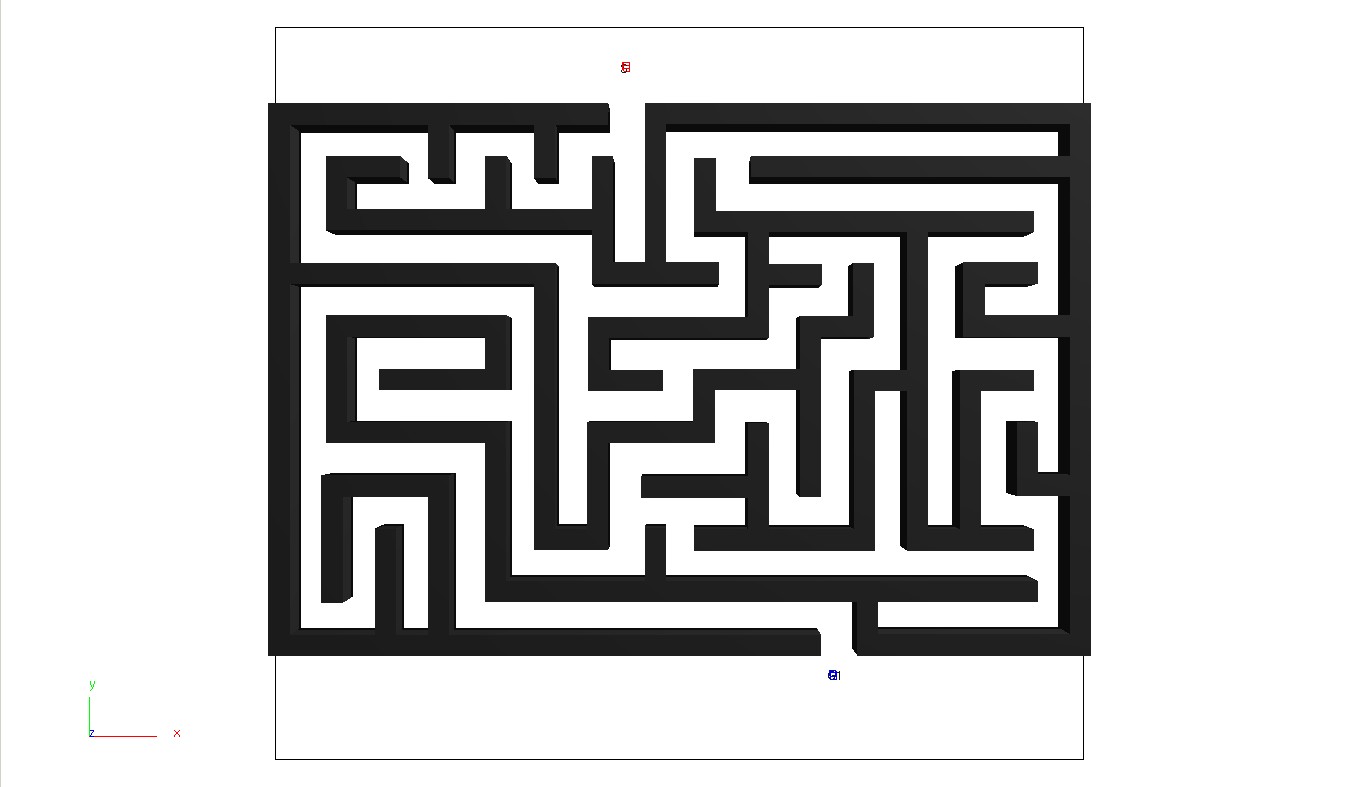}%
        \caption{An illustration of the experiment described in \secref{maze}. The start and goal configurations are colored blue and red, respectively.}
        \figlab{maze}
    \end{figure}

    \smallskip%
    \item \itemlab{long:detour}%
    \LongDetour (\figref{path_length}).
    \begin{itemize}
        \item Environment: A 2D bounding box of size $100\times 20$ with a wall crossing it from left to right. The wall has two passages, a narrow passage close to the left end, and a wider passage between the wall and the right end of the bounding box.

        \item Robot: \DOFX{3} rectangular robot

        \item Task: Move from a starting position on the left side of the environment on one side of the wall, to a configuration on the left side on the other side of the wall.

        \item Algorithm: \RRT

    \end{itemize}
    We used this input to demonstrate that our strategies lead to shorter solutions, see \tblref{path_length}.

    \begin{figure}[h!]
           \centering \subfloat[]{%
              \includegraphics[width=0.45\linewidth]{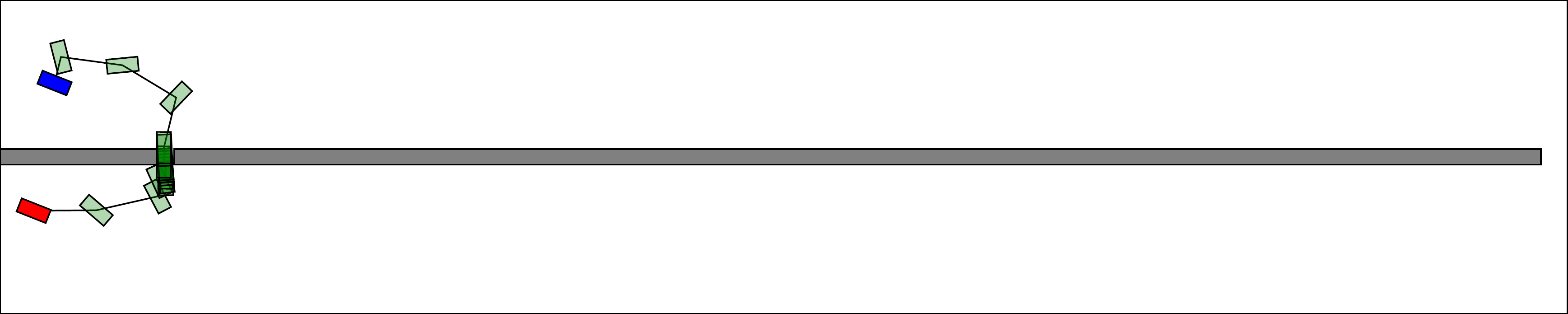}%
           }%
           \hfil%
           \subfloat[]{
              \includegraphics[width=0.45\linewidth]{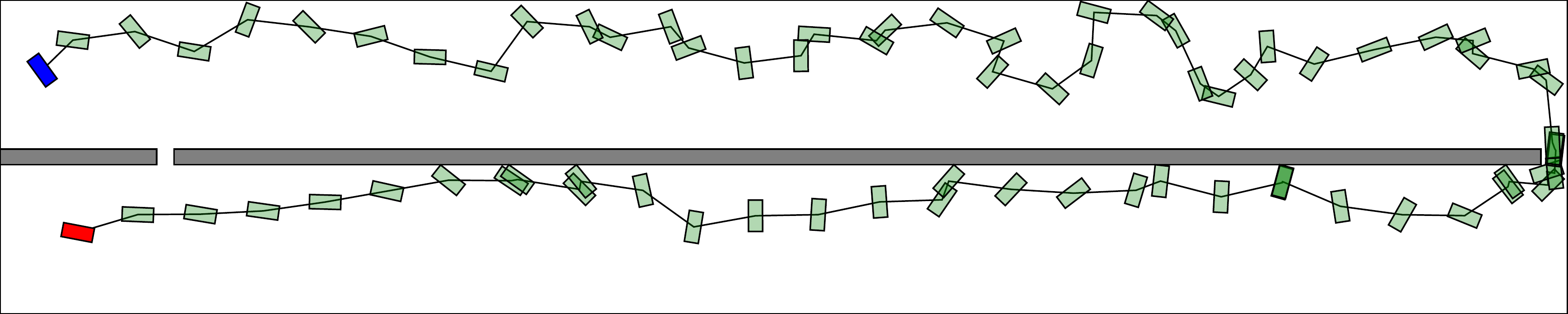}%
           }%
        \caption{An illustration of the experiment described in \secref{path_length}. Two obstacles are shown in gray. Note that the top gap is bigger and thus easier for the motion planner to discover if the exploration reaches this region, while the bottom gap is narrower and thus harder to discover.}
        \figlab{path_length}
    \end{figure}
\end{compactenumI}

\subsection{High variance experiment}

\subsubsection{Basic benchmark}

\seclab{simple:passage}

The first experiment, run on input \SimplePassageRef, serves as a comparison benchmark for the rest of the section. We run all of the strategies described above, as well as a single thread \RRT for comparison.  The simplicity of this input is intended to demonstrate the strategies without other considerations getting in the way, such as the difficulty of the input itself or unexpected behaviors of the motion planner that arise over more complicated inputs.  To that end, we run this experiment on \ValisRef.

This experiment includes a simple narrow passage environment and a basic \DOFX{6} free-flying object, see \SimplePassageRef, which are enough to create a high-variance instance of the motion planning problem, as a ``lucky'' sample with the right orientation may result in an almost immediate solution, but, if some configurations with a ``bad'' orientation are sampled next to the passage's opening, catastrophic runtimes may be encountered.

\begin{table*}[h!]
	\centering
	\begin{tabular}{c}
\begin{tabular}{|l||r|r|r||r|r||r|r|r|}
  \hline
  \textbf{Simulation} & \textbf{\#} & \textbf{Succs} & \textbf{Fails} & \textbf{Mean} & \textbf{Median} & \textbf{Std Dev} & \textbf{Min} & \textbf{Max} \\\hline
  \cellcolor{lightgray}{\texttt{Single}} & \cellcolor{lightgray}{\texttt{100}} & \cellcolor{lightgray}{\texttt{67}} & \cellcolor{lightgray}{\FailX{33}} & \cellcolor{lightgray}{\texttt{840.82}} & \cellcolor{lightgray}{\texttt{637.00}} & \cellcolor{lightgray}{\texttt{770.50}} & \cellcolor{lightgray}{\texttt{1.00}} & \cellcolor{lightgray}{\texttt{2963.00}} \\
  \texttt{Parallel} & \texttt{100} & \texttt{100} & \texttt{0} & \texttt{117.49} & \texttt{26.50} & \texttt{234.22} & \texttt{1.00} & \texttt{1509.00} \\
\hline
  \cellcolor{lightgray}{\texttt{Wide}} & \cellcolor{lightgray}{\texttt{100}} & \cellcolor{lightgray}{\texttt{100}} & \cellcolor{lightgray}{\texttt{0}} & \cellcolor{lightgray}{\texttt{24.93}} & \cellcolor{lightgray}{\texttt{16.00}} & \cellcolor{lightgray}{\texttt{29.83}} & \cellcolor{lightgray}{\texttt{1.00}} & \cellcolor{lightgray}{\texttt{145.00}} \\
  \texttt{Random} & \texttt{100} & \texttt{100} & \texttt{0} & \texttt{20.10} & \texttt{12.50} & \texttt{21.97} & \texttt{1.00} & \texttt{99.00} \\
  \cellcolor{lightgray}{\texttt{Counter+Cache}} & \cellcolor{lightgray}{\texttt{100}} & \cellcolor{lightgray}{\texttt{100}} & \cellcolor{lightgray}{\texttt{0}} & \cellcolor{lightgray}{\texttt{19.76}} & \cellcolor{lightgray}{\texttt{11.50}} & \cellcolor{lightgray}{\texttt{20.64}} & \cellcolor{lightgray}{\texttt{1.00}} & \cellcolor{lightgray}{\texttt{108.00}} \\
  \texttt{Random $\zeta(2)$} & \texttt{100} & \texttt{100} & \texttt{0} & \texttt{18.11} & \texttt{12.00} & \texttt{21.82} & \texttt{1.00} & \texttt{125.00} \\
  \cellcolor{lightgray}{\texttt{Counter}} & \cellcolor{lightgray}{\texttt{100}} & \cellcolor{lightgray}{\texttt{100}} & \cellcolor{lightgray}{\texttt{0}} & \cellcolor{lightgray}{\texttt{14.74}} & \cellcolor{lightgray}{\texttt{11.00}} & \cellcolor{lightgray}{\texttt{16.42}} & \cellcolor{lightgray}{\texttt{1.00}} & \cellcolor{lightgray}{\texttt{110.00}} \\\hline
\end{tabular}

		(a) \SimplePassageRef results on \ValisRef. $\Bigl.$
	\end{tabular}
        \caption{Table summarizing the runtime results of the experiment described in \secref{simple:passage}, the benchmark experiment, on \ValisRef.}
     \tbllab{simple:passage}
\end{table*}

\myparagraph{Results and discussion.} %
The results, summarized in \tblref{simple:passage}, show the super-linear improvement of using additional threads in high-variance motion planning problems. The high variance phenomenon is also apparent in the high failure rate ($33\%$) of the single thread method (i.e., the runtime exceeded 3000 seconds in a single try).

The remarkable advantage of the different restarting strategies over the parallel-OR model is also evident, with the slowest restarting strategy (wide search) giving an average runtime almost $5$ times faster than that of the parallel strategy. The large variance and gap between the min and the max in \tblref{simple:passage} are further evidence of the high variance.

\subsubsection{Fixed base manipulator}
\seclab{shelf}

In this experiment, run on the \ManipulatorRef input, a manipulator resembling a UR5 tasked with reaching into a cluttered shelf serves two purposes. On top of showcasing the advantage of the restarting strategies over the parallel-OR model, it is the only experiment where we demonstrate the effectiveness of the restarting strategies on a ``plain'' single process architecture, and the difference between fixed-\TTL strategies to ours. Recall that the optimal strategy is always a fixed-\TTL strategy. See \obsref{fixed:ttl:optimal}.

Intuitively, this motion planning problem, described below in greater detail, contains a bottleneck in the form of a small set of configurations that must be sampled or found for a valid motion to be generated (known as a narrow passage) as the manipulator must find its way between the objects on the shelf.

\begin{table*}[h!]
    \centering
    \begin{tabular}{c}
\begin{tabular}{|l||r|r|r||r|r||r|r|r|r}
  \hline
  \textbf{Simulation} & \textbf{\#} & \textbf{Succs} & \textbf{Fails} & \textbf{Mean} & \textbf{Median} & \textbf{Std Dev} & \textbf{Min} & \textbf{Max} \\\hline
  \cellcolor{lightgray}{\texttt{Single}} & \cellcolor{lightgray}{\texttt{100}} & \cellcolor{lightgray}{\texttt{7}} & \cellcolor{lightgray}{\FailX{93}} & \cellcolor{lightgray}{\texttt{898.14}} & \cellcolor{lightgray}{\texttt{312.00}} & \cellcolor{lightgray}{\texttt{1074.23}} & \cellcolor{lightgray}{\texttt{34.00}} & \cellcolor{lightgray}{\texttt{2956.00}} \\
\hline
  \texttt{Random $\zeta(2)$} & \texttt{100} & \texttt{19} & \FailX{81} & \texttt{1199.63} & \texttt{907.00} & \texttt{1034.27} & \texttt{2.00} & \texttt{2933.00} \\
  \cellcolor{lightgray}{\texttt{Random}} & \cellcolor{lightgray}{\texttt{100}} & \cellcolor{lightgray}{\texttt{30}} & \cellcolor{lightgray}{\FailX{70}} & \cellcolor{lightgray}{\texttt{1535.50}} & \cellcolor{lightgray}{\texttt{1527.00}} & \cellcolor{lightgray}{\texttt{749.81}} & \cellcolor{lightgray}{\texttt{149.00}} & \cellcolor{lightgray}{\texttt{2987.00}} \\
  \texttt{Counter} & \texttt{100} & \texttt{37} & \FailX{63} & \texttt{1208.51} & \texttt{1171.00} & \texttt{814.20} & \texttt{18.00} & \texttt{2768.00} \\
  \cellcolor{lightgray}{\texttt{Wide}} & \cellcolor{lightgray}{\texttt{100}} & \cellcolor{lightgray}{\texttt{40}} & \cellcolor{lightgray}{\FailX{60}} & \cellcolor{lightgray}{\texttt{1422.03}} & \cellcolor{lightgray}{\texttt{1431.50}} & \cellcolor{lightgray}{\texttt{753.52}} & \cellcolor{lightgray}{\texttt{88.00}} & \cellcolor{lightgray}{\texttt{2967.00}} \\
  \texttt{Counter+Cache} & \texttt{100} & \texttt{45} & \FailX{55} & \texttt{1362.18} & \texttt{1219.00} & \texttt{861.28} & \texttt{3.00} & \texttt{2849.00} \\\hline
\end{tabular}

      (a) \ManipulatorRef results on \SingleRef. $\Bigl.$
      \\
      \\
      %
\begin{tabular}{|l||r|r|r||r|r||r|r|r|r}
  \hline
  \textbf{Simulation} & \textbf{\#} & \textbf{Succs} & \textbf{Fails} & \textbf{Mean} & \textbf{Median} & \textbf{Std Dev} & \textbf{Min} & \textbf{Max} \\\hline
  \cellcolor{lightgray}{\texttt{Parallel}} & \cellcolor{lightgray}{\texttt{42}} & \cellcolor{lightgray}{\texttt{25}} & \cellcolor{lightgray}{\FailX{17}} & \cellcolor{lightgray}{\texttt{733.44}} & \cellcolor{lightgray}{\texttt{458.00}} & \cellcolor{lightgray}{\texttt{632.32}} & \cellcolor{lightgray}{\texttt{46.00}} & \cellcolor{lightgray}{\texttt{2220.00}} \\
  \texttt{Wide} & \texttt{100} & \texttt{99} & \FailX{1} & \texttt{606.43} & \texttt{398.00} & \texttt{657.69} & \texttt{1.00} & \texttt{2972.00} \\
\hline
  \cellcolor{lightgray}{\texttt{Random}} & \cellcolor{lightgray}{\texttt{100}} & \cellcolor{lightgray}{\texttt{100}} & \cellcolor{lightgray}{\texttt{0}} & \cellcolor{lightgray}{\texttt{541.49}} & \cellcolor{lightgray}{\texttt{365.50}} & \cellcolor{lightgray}{\texttt{513.54}} & \cellcolor{lightgray}{\texttt{3.00}} & \cellcolor{lightgray}{\texttt{2283.00}} \\
  \texttt{Random $\zeta(2)$} & \texttt{100} & \texttt{99} & \FailX{1} & \texttt{477.57} & \texttt{308.00} & \texttt{486.30} & \texttt{1.00} & \texttt{1954.00} \\
  \cellcolor{lightgray}{\texttt{Counter}} & \cellcolor{lightgray}{\texttt{100}} & \cellcolor{lightgray}{\texttt{100}} & \cellcolor{lightgray}{\texttt{0}} & \cellcolor{lightgray}{\texttt{462.92}} & \cellcolor{lightgray}{\texttt{268.00}} & \cellcolor{lightgray}{\texttt{549.91}} & \cellcolor{lightgray}{\texttt{3.00}} & \cellcolor{lightgray}{\texttt{2802.00}} \\
  \texttt{Counter+Cache} & \texttt{100} & \texttt{100} & \texttt{0} & \texttt{368.49} & \texttt{284.00} & \texttt{333.42} & \texttt{7.00} & \texttt{2168.00} \\\hline
\end{tabular}

      \\
      (b) \ManipulatorRef results on \ValisRef. $\Bigl.$
      \\
      \\
      %
\begin{tabular}{|l||r||r|r||r|r|r|}
  \hline
  \textbf{Simulation} & \textbf{\#} & \textbf{Mean} & \textbf{Median} & \textbf{Std Dev} & \textbf{Min} & \textbf{Max} \\\hline
  \cellcolor{lightgray}{\texttt{Parallel}} & \cellcolor{lightgray}{\texttt{100}} & \cellcolor{lightgray}{\texttt{154.97}} & \cellcolor{lightgray}{\texttt{39.50}} & \cellcolor{lightgray}{\texttt{271.33}} & \cellcolor{lightgray}{\texttt{5.00}} & \cellcolor{lightgray}{\texttt{1400.00}} \\
\hline
  \texttt{Wide} & \texttt{100} & \texttt{115.64} & \texttt{61.00} & \texttt{159.18} & \texttt{4.00} & \texttt{1096.00} \\
  \cellcolor{lightgray}{\texttt{Random $\zeta(2)$}} & \cellcolor{lightgray}{\texttt{100}} & \cellcolor{lightgray}{\texttt{64.69}} & \cellcolor{lightgray}{\texttt{41.50}} & \cellcolor{lightgray}{\texttt{68.92}} & \cellcolor{lightgray}{\texttt{5.00}} & \cellcolor{lightgray}{\texttt{315.00}} \\
  \texttt{Counter+Cache} & \texttt{100} & \texttt{64.33} & \texttt{53.00} & \texttt{61.56} & \texttt{4.00} & \texttt{364.00} \\
  \cellcolor{lightgray}{\texttt{Counter}} & \cellcolor{lightgray}{\texttt{100}} & \cellcolor{lightgray}{\texttt{58.42}} & \cellcolor{lightgray}{\texttt{39.00}} & \cellcolor{lightgray}{\texttt{53.90}} & \cellcolor{lightgray}{\texttt{5.00}} & \cellcolor{lightgray}{\texttt{234.00}} \\
  \texttt{Random} & \texttt{100} & \texttt{47.54} & \texttt{27.00} & \texttt{59.79} & \texttt{4.00} & \texttt{442.00} \\
\hline
  \cellcolor{lightgray}{\texttt{\TTL $2$}} & \cellcolor{lightgray}{\texttt{100}} & \cellcolor{lightgray}{\texttt{51.76}} & \cellcolor{lightgray}{\texttt{38.00}} & \cellcolor{lightgray}{\texttt{50.23}} & \cellcolor{lightgray}{\texttt{4.00}} & \cellcolor{lightgray}{\texttt{237.00}} \\
  \texttt{\TTL $3$} & \texttt{100} & \texttt{47.94} & \texttt{27.00} & \texttt{49.57} & \texttt{5.00} & \texttt{200.00} \\
  \cellcolor{lightgray}{\texttt{\TTL $5$}} & \cellcolor{lightgray}{\texttt{100}} & \cellcolor{lightgray}{\texttt{50.51}} & \cellcolor{lightgray}{\texttt{34.00}} & \cellcolor{lightgray}{\texttt{56.48}} & \cellcolor{lightgray}{\texttt{4.00}} & \cellcolor{lightgray}{\texttt{303.00}} \\
  \texttt{\TTL $10$} & \texttt{100} & \texttt{53.13} & \texttt{38.50} & \texttt{46.54} & \texttt{5.00} & \texttt{272.00} \\\hline
\end{tabular}

      \\
      (c) \ManipulatorRef results on \ClusterRef. $\Bigl.$
    \end{tabular}
    \caption{Tables summarizing the runtime results of the experiment
       described in \secref{shelf} on the different platforms
       \SingleRef, \ValisRef, and \ClusterRef.}
    \tbllab{shelf}
\end{table*}

\myparagraph{Results and discussion.} %
The results, summarized in \tblref{shelf}, show the superiority of the different restarting strategies in all scenarios. In \tblref{shelf} (a), we see that a simple run of \RRT trying to solve the shelf motion planning task has a success rate of between $6\% - 36\%$ that of the restarting strategies. The difference between the restarting strategies is also remarkable, with the worst performing strategy, random $\zeta(2)$, achieving a success rate only $\sim 42\%$ that of the top performer, counter + cache. The descriptive statistics other than the success rate are meaningless in this case.

In \tblref{shelf} (b) and (c), we see that as the number of threads increases, the differences between the non-fixed strategies shrink. The mean and median runtimes are relatively similar on both \ValisRef and \ClusterRef machines, and no single non-fixed strategy shows an unambiguous advantage over the others. The difference between the parallel-OR model and these strategies remains quite large. On the \ValisRef machine, we stopped the experiments with the parallel strategy to save time, after it became clear that it would not produce meaningful statistical results due to low success rates, and on the \ClusterRef machine, it maintains \RRT's high variance runtime distribution, and achieves substantially slower runtimes on average.

In \tblref{shelf} (c), we see that fixed-\TTL strategies perform better on this problem as the threshold approaches 3, surpassing the non-fixed strategies. This conforms with the theoretical analyses and calculations we performed by following \lemref{full:k} and using an empirical distribution to approximate the ground truth $\Distrib$ of the runtime of \RRT on this instance.

\subsubsection{Mobile robot}
\seclab{bug:trap}

In this experiment, run on the \BugTrapRef input, we are tasked with moving a free-flying cylindrical object through a narrow passage and into a ``bug trap". The purpose of this experiment is to examine the performance of the restarting strategies on a motion planning instance with a mobile robot.

\begin{table*}[h!]
    \centering
    \begin{tabular}{c}
      %
\begin{tabular}{|l||r|r|r||r|r||r|r|r|}
  \hline
  \textbf{Simulation} & \textbf{\#} & \textbf{Succs} & \textbf{Fails} & \textbf{Mean} & \textbf{Median} & \textbf{Std Dev} & \textbf{Min} & \textbf{Max} \\\hline
  \cellcolor{lightgray}{\texttt{Parallel}} & \cellcolor{lightgray}{\texttt{61}} & \cellcolor{lightgray}{\texttt{0}} & \cellcolor{lightgray}{\FailX{61}} & \cellcolor{lightgray}{\texttt{---}} & \cellcolor{lightgray}{\texttt{---}} & \cellcolor{lightgray}{\texttt{---}} & \cellcolor{lightgray}{\texttt{---}} & \cellcolor{lightgray}{\texttt{---}} \\
\hline
  \texttt{Counter+Cache} & \texttt{100} & \texttt{57} & \FailX{43} & \texttt{852.95} & \texttt{602.00} & \texttt{763.67} & \texttt{12.00} & \texttt{2994.00} \\
  \cellcolor{lightgray}{\texttt{Counter}} & \cellcolor{lightgray}{\texttt{100}} & \cellcolor{lightgray}{\texttt{61}} & \cellcolor{lightgray}{\FailX{39}} & \cellcolor{lightgray}{\texttt{1148.41}} & \cellcolor{lightgray}{\texttt{966.00}} & \cellcolor{lightgray}{\texttt{933.65}} & \cellcolor{lightgray}{\texttt{11.00}} & \cellcolor{lightgray}{\texttt{2865.00}} \\
  \texttt{Random $\zeta(2)$} & \texttt{100} & \texttt{62} & \FailX{38} & \texttt{1155.65} & \texttt{967.00} & \texttt{920.35} & \texttt{3.00} & \texttt{2946.00} \\\hline
\end{tabular}

      (a) \BugTrapRef results on \ValisRef. $\Bigl.$
      \\
      \\
      %
\begin{tabular}{|l||r|r|r||r|r||r|r|r|}
  \hline
  \textbf{Simulation} & \textbf{\#} & \textbf{Succs} & \textbf{Fails} & \textbf{Mean} & \textbf{Median} & \textbf{Std Dev} & \textbf{Min} & \textbf{Max} \\\hline
  \cellcolor{lightgray}{\texttt{Parallel}} & \cellcolor{lightgray}{\texttt{100}} & \cellcolor{lightgray}{\texttt{36}} & \cellcolor{lightgray}{\FailX{64}} & \cellcolor{lightgray}{\texttt{1526.69}} & \cellcolor{lightgray}{\texttt{1427.50}} & \cellcolor{lightgray}{\texttt{904.55}} & \cellcolor{lightgray}{\texttt{6.00}} & \cellcolor{lightgray}{\texttt{2925.00}} \\
\hline
  \texttt{Wide} & \texttt{100} & \texttt{68} & \FailX{32} & \texttt{787.65} & \texttt{503.00} & \texttt{826.54} & \texttt{5.00} & \texttt{2899.00} \\
  \cellcolor{lightgray}{\texttt{Random}} & \cellcolor{lightgray}{\texttt{100}} & \cellcolor{lightgray}{\texttt{100}} & \cellcolor{lightgray}{\texttt{0}} & \cellcolor{lightgray}{\texttt{446.01}} & \cellcolor{lightgray}{\texttt{209.50}} & \cellcolor{lightgray}{\texttt{543.65}} & \cellcolor{lightgray}{\texttt{6.00}} & \cellcolor{lightgray}{\texttt{2519.00}} \\
  \texttt{Counter+Cache} & \texttt{100} & \texttt{100} & \texttt{0} & \texttt{346.37} & \texttt{223.50} & \texttt{349.84} & \texttt{5.00} & \texttt{2084.00} \\
  \cellcolor{lightgray}{\texttt{Counter}} & \cellcolor{lightgray}{\texttt{100}} & \cellcolor{lightgray}{\texttt{100}} & \cellcolor{lightgray}{\texttt{0}} & \cellcolor{lightgray}{\texttt{315.14}} & \cellcolor{lightgray}{\texttt{190.50}} & \cellcolor{lightgray}{\texttt{325.23}} & \cellcolor{lightgray}{\texttt{5.00}} & \cellcolor{lightgray}{\texttt{1755.00}} \\
  \texttt{Random $\zeta(2)$} & \texttt{100} & \texttt{100} & \texttt{0} & \texttt{275.54} & \texttt{155.50} & \texttt{303.83} & \texttt{5.00} & \texttt{1392.00} \\\hline
\end{tabular}

      \\
      (b) \BugTrapRef results on \ClusterRef. $\Bigl.$
    \end{tabular}
    \caption{Tables summarizing the runtime results of the experiment described in \secref{bug:trap} on the different platforms \ValisRef and \ClusterRef.}
    \tbllab{bug:trap}
\end{table*}

\myparagraph{Results and discussion.} %
\tblref{bug:trap}, summarizing the results of this experiment, shows again the clear advantage of all restarting strategies over the parallel-OR model, whose runs on \ValisRef we, again, stopped after it became clear that the low success rate would render the statistical results insignificant. Wide search stands out in this experiment, with relatively poor results. We hypothesize that due to the difficulty of the problem (for \RRT), the number of paused processes quickly became too large for the machine, even \ClusterRef, to efficiently handle, thus leading to poor results.  In particular, the hardware of \ClusterRef seems to be unfriendly to the multitude of processes that the wide search creates, and seems to underperform for this strategy. This phenomenon is also evident in our other experiments.

\subsection{Low variance experiment}
\seclab{maze}

In this experiment, run on the \MazeRef input, we have a motion planning with a stable runtime distribution for \SBMP algorithms. Also, the probability of finding a solution within a few seconds is essentially zero, meaning that the optimal \TTL is high relatively to previous experiments.

This is the only experiment in which we use \PRM, as it is more suitable for this problem, and enabled us to run a full set of experiments much quicker. The purpose of this experiment is to emphasize that the restarting strategies' advantage stems directly from the high variance of the problem, completing our presentation of the general framework.

\begin{table*}[h!]
	\centering
	\begin{tabular}{c}
          %
\begin{tabular}{|l||r||r|r||r|r|r|}
  \hline
  \textbf{Simulation} & \textbf{\#} & \textbf{Mean} & \textbf{Median} & \textbf{Std Dev} & \textbf{Min} & \textbf{Max} \\\hline
  \cellcolor{lightgray}{\texttt{Parallel}} & \cellcolor{lightgray}{\texttt{100}} & \cellcolor{lightgray}{\texttt{124.77}} & \cellcolor{lightgray}{\texttt{124.00}} & \cellcolor{lightgray}{\texttt{19.82}} & \cellcolor{lightgray}{\texttt{85.00}} & \cellcolor{lightgray}{\texttt{183.00}} \\
\hline
  \texttt{Counter} & \texttt{100} & \texttt{318.64} & \texttt{324.00} & \texttt{82.94} & \texttt{173.00} & \texttt{562.00} \\
  \cellcolor{lightgray}{\texttt{Random $\zeta(2)$}} & \cellcolor{lightgray}{\texttt{100}} & \cellcolor{lightgray}{\texttt{247.45}} & \cellcolor{lightgray}{\texttt{221.00}} & \cellcolor{lightgray}{\texttt{113.12}} & \cellcolor{lightgray}{\texttt{86.00}} & \cellcolor{lightgray}{\texttt{901.00}} \\
  \texttt{Counter+Cache} & \texttt{100} & \texttt{220.96} & \texttt{225.00} & \texttt{48.79} & \texttt{120.00} & \texttt{325.00} \\
  \cellcolor{lightgray}{\texttt{Random}} & \cellcolor{lightgray}{\texttt{100}} & \cellcolor{lightgray}{\texttt{214.43}} & \cellcolor{lightgray}{\texttt{194.50}} & \cellcolor{lightgray}{\texttt{93.33}} & \cellcolor{lightgray}{\texttt{95.00}} & \cellcolor{lightgray}{\texttt{542.00}} \\
  \texttt{Wide} & \texttt{100} & \texttt{165.01} & \texttt{157.00} & \texttt{36.40} & \texttt{96.00} & \texttt{286.00} \\\hline
\end{tabular}

          (a) \MazeRef results on \ValisRef. $\Bigl.$
		\\
		\\
          %
\begin{tabular}{|l||r||r|r||r|r|r|}
  \hline
  \textbf{Simulation} & \textbf{\#} & \textbf{Mean} & \textbf{Median} & \textbf{Std Dev} & \textbf{Min} & \textbf{Max} \\\hline
  \cellcolor{lightgray}{\texttt{Parallel}} & \cellcolor{lightgray}{\texttt{100}} & \cellcolor{lightgray}{\texttt{98.23}} & \cellcolor{lightgray}{\texttt{101.00}} & \cellcolor{lightgray}{\texttt{8.68}} & \cellcolor{lightgray}{\texttt{76.00}} & \cellcolor{lightgray}{\texttt{114.00}} \\
\hline
  \texttt{Wide} & \texttt{100} & \texttt{150.18} & \texttt{150.00} & \texttt{17.42} & \texttt{105.00} & \texttt{193.00} \\
  \cellcolor{lightgray}{\texttt{Counter}} & \cellcolor{lightgray}{\texttt{100}} & \cellcolor{lightgray}{\texttt{145.18}} & \cellcolor{lightgray}{\texttt{143.50}} & \cellcolor{lightgray}{\texttt{21.36}} & \cellcolor{lightgray}{\texttt{102.00}} & \cellcolor{lightgray}{\texttt{191.00}} \\
  \texttt{Random $\zeta(2)$} & \texttt{100} & \texttt{137.82} & \texttt{135.50} & \texttt{26.55} & \texttt{86.00} & \texttt{227.00} \\
  \cellcolor{lightgray}{\texttt{Counter+Cache}} & \cellcolor{lightgray}{\texttt{100}} & \cellcolor{lightgray}{\texttt{135.66}} & \cellcolor{lightgray}{\texttt{132.00}} & \cellcolor{lightgray}{\texttt{20.59}} & \cellcolor{lightgray}{\texttt{95.00}} & \cellcolor{lightgray}{\texttt{211.00}} \\\hline
\end{tabular}

          \\
          (b) \MazeRef results on \ClusterRef. $\Bigl.$
	\end{tabular}
        \caption{Tables summarizing the runtime results of the experiment described in \secref{maze} on the different platforms \ValisRef and \ClusterRef.}
\tbllab{maze}
\end{table*}

\myparagraph{Results and discussion.} %
The results, summarized in \tblref{maze}, show that, as expected, the parallel-OR method outperforms all of the restarting strategies. This fits the theoretical analysis, as the parallel strategy does not waste time on short runs of \PRM, and its runs reach more advanced stages of \PRM's run in which the success probability is meaningful.

On \ValisRef we also see that the high optimal \TTL favors the random strategies random counter and random $\zeta(2)$ over the deterministic counter search, since relatively higher \TTL{}s may appear earlier in the execution of the strategy. Also, counter + cache achieves relatively good results, as it conserves computation towards the later stages of the execution. These effects can also be seen on \ClusterRef, but to a lesser extent. Wide search's disappointing performance on \ClusterRef, mentioned above, is also evident here.

\subsection{Path length experiment}
\seclab{path_length}

This last experiment, run on the \LongDetourpRef input, includes a somewhat contrived environment, designed especially to contain two paths, one that is difficult to find using uniform sampling, but that results in a short path, and one that can be found given a reasonable amount of time, but results in a long path. Our goal here is to show that under such circumstances, i.e., when a short path that's hard to detect exists, the restarting strategies yield (in expectation) better solutions than the parallel-OR strategy.

In this experiment, we measure and report only the \textbf{lengths} of the paths returned by the algorithm. All of the numbers in \tblref{path_length} refer only to that quantity.

\begin{table*}[h!]
    \centering
\begin{tabular}{|l||r||r|r||r|r|r|}
  \cline{3-7}
  \multicolumn{2}{c||}{}
  &\multicolumn{5}{c|}{Path lengths}
  \\
  \hline
  \textbf{Simulation} & \textbf{\#} & \textbf{Mean} & \textbf{Median} & \textbf{Std Dev} & \textbf{Min} & \textbf{Max} \\\hline
  \cellcolor{lightgray}{\texttt{Parallel}} & \cellcolor{lightgray}{\texttt{100}} & \cellcolor{lightgray}{\texttt{208.30}} & \cellcolor{lightgray}{\texttt{232.12}} & \cellcolor{lightgray}{\texttt{68.69}} & \cellcolor{lightgray}{\texttt{23.78}} & \cellcolor{lightgray}{\texttt{258.97}} \\
  \texttt{Random} & \texttt{100} & \texttt{148.53} & \texttt{226.89} & \texttt{100.20} & \texttt{22.47} & \texttt{249.52} \\
  \cellcolor{lightgray}{\texttt{Random $\zeta(2)$}} & \cellcolor{lightgray}{\texttt{100}} & \cellcolor{lightgray}{\texttt{148.18}} & \cellcolor{lightgray}{\texttt{226.34}} & \cellcolor{lightgray}{\texttt{99.37}} & \cellcolor{lightgray}{\texttt{24.50}} & \cellcolor{lightgray}{\texttt{255.39}} \\
  \texttt{Counter} & \texttt{100} & \texttt{127.58} & \texttt{45.71} & \texttt{100.08} & \texttt{21.63} & \texttt{249.36} \\\hline
\end{tabular}

    \medskip%
    \caption{\LongDetourpRef results: The results of the experiment described in \secref{path_length}.  Unlike the other tables, all the numbers are \textbf{lengths} of paths, not running times.  }
    \tbllab{path_length}%
\end{table*}

\myparagraph{Results and discussion.} %
In the results, summarized in \tblref{path_length}, we see, by the mean value of path lengths, that all restarting strategies found the shorter path with higher frequency than the parallel-OR model. The counter search strategy, as reflected in its median path length value, has found the shorter path more often than not.

\section{How to use our implementation (for \SBMP)}
\seclab{applying:in:sbmp}

Stochastic restarting strategies are practical and impactful tools that can enhance the performance of \SBMP algorithms. Below, we provide recommendations for applying these methods.

\myparagraph{How to use.} %
We provide an open source implementation of the framework \cite{ah-csula-25-real}. We emphasize that converting an existing program $P$ to use the framework is quite easy. One needs to create a simple wrapper script that: \smallskip%
\begin{compactenumi}
    \item copies the input to a specified directory,

    \item run $P$ on the copied input (storing any intermediate files in this directory), and

    \item if the run succeeds, the script creates a text file (in a specified place) indicating the success.
\end{compactenumi}
\smallskip%
Thus, using our framework can be implemented in less than 20 minutes.  In particular, it might be useful as a tool for harnessing parallelism in a multi-core machine, in a situation where the existing implementation is single-threaded.

\myparagraph{When to use.} %
Generally speaking, the speedup techniques shine when the descriptive complexity of the solution is ``small'' -- the optimal solution is a simple motion involving few parts. Conversely, the speedup technique fails miserably if the motion is complicated (such as in the \MazeRef input, see \tblref{maze}). If using $N$ threads, using $N-1$ with restarting strategies and the remaining thread with ``plain'' \RRT avoids catastrophic runtimes, while taking into account the possibility of a low-variance instance.

\section{Conclusions}
\seclab{conclusions}

In this paper, we have demonstrated the applicability of the stochastic resetting framework to \SBMP algorithms. We have provided a comprehensive introduction to the theory and intuition behind it, and experimentally verified its effectiveness in motion planning instances that give rise to high-variance runtime distributions. We have also presented several novel restarting strategies that are asymptotically optimal and easy to parallelize.

Our results suggest that, when possible, parallel implementations of \SBMP algorithms should invest computational resources into processes running both plain implementations of the algorithm as well as restarting strategies, greatly reducing the probability and severity of catastrophic runtimes, while maintaining progress towards a solution.

\myparagraph{Future research.}
A natural direction is to try and break complicated motion planning problems into short chunks where our techniques help, and combine the solutions to get a global complicated solution faster.

\printbibliography

\end{document}